\documentclass[lettersize,journal]{IEEEtran}
\usepackage{amsmath,amsfonts}
\usepackage{algorithm}
\usepackage{array}
\usepackage[caption=false,font=normalsize,labelfont=sf,textfont=sf]{subfig}
\usepackage{textcomp}
\usepackage{stfloats}
\usepackage{url}
\usepackage{verbatim}
\usepackage{graphicx}
\usepackage{cite}

\usepackage{microtype}
\usepackage{graphicx}
\usepackage{booktabs}
\usepackage{hyperref}

\usepackage{amsmath}
\usepackage{amssymb}
\usepackage{mathtools}
\usepackage{amsthm}
\usepackage[capitalize,noabbrev]{cleveref}
\usepackage{graphicx}
\usepackage{amsmath,amssymb,amsfonts}
\usepackage{booktabs}
\usepackage{amsthm}
\usepackage[noend]{algpseudocode}
\usepackage{epstopdf}
\usepackage{multirow}
\usepackage{bm}
\usepackage{color}
\usepackage{wrapfig}
\usepackage[textsize=tiny]{todonotes}
\usepackage{subfig}

\newtheorem{theorem}{Theorem}%[section]
\newtheorem{corollary}{Corollary}%[section]
\newtheorem{remark}{Remark}%[section]
\newtheorem{exam}{Example}%[section]
\newtheorem{proposition}{Proposition}%[section]
\newtheorem{definition}{Definition}%[section]

\newcommand{\thmref}[1]{Theorem~\ref{#1}}

\newcommand{\corref}[1]{Corollary~\ref{#1}}
\newcommand{\propref}[1]{Proposition~\ref{#1}}

\newcommand{\NN}{\mathbb{N}}
\newcommand{\RR}{\mathbb{R}}
\newcommand{\ZZ}{\mathbb{Z}}

\def\b{\bm}
\def\f{\frac}

\def\sph2{\mathbb{S}^{2}}
\def\sp1{\mathbb{S}^{1}}

\def\RR{\mathbb{R}}

\def\NN{\mathbb{N}}
\def\p{\mathcal{P}}
\def\g{\mathcal{G}}
\def\v{\mathcal{V}}
\def\w{\mathcal{E}}
\def\c{\mathcal{C}}

\def\d{\mathcal{D}}

\def\supp{\mathrm{supp}}
\def\dim{\mathrm{dim}}
\def\Span{\mathrm{span}}

\def\rank{\mathrm{Rank}}
\def\A{\mathcal{A}}
\def\per{\pi}
\def\ab{\allowbreak}

\begin{document}
%\linenumbers
%\pagewiselinenumbers   %每页开始重新从1开始编号
%\switchlinenumbers  %编号在左栏的左侧和右栏的右侧
\title{Permutation Equivariant Graph Framelets for\\ Heterophilous Graph Learning}

%\author{Jianfei Li$^*$, \thanks{Department of Mathematics, City University of Hong Kong (\texttt{jianfeili2-c@my.cityu.edu.hk})}
%	\and
%	Ruigang Zheng$^*$,\thanks{Department of Mathematics, City University of Hong Kong (\texttt{ruigzheng2-c@my.cityu.edu.hk})}
%	\and
%	Han Feng,\thanks{Department of Mathematics, City University of Hong Kong (\texttt{hanfeng@cityu.edu.hk})}
%	\and 
%	Xiaosheng Zhuang,\thanks{Department of Mathematics, City University of Hong Kong (\texttt{xzhuang7@cityu.edu.hk})}}

\author{Jianfei Li$^\dagger$,~
        Ruigang Zheng$^\dagger$,~
        Han Feng,~
        Ming~Li*,~\IEEEmembership{Member,~IEEE,}
        Xiaosheng Zhuang*
\thanks{J. Li, R. Zheng, H. Feng, and X. Zhuang are with the Department of Mathematics, City University of Hong Kong, Hong
Kong, China. (Emails: jianfeili2-c@my.cityu.edu.hk, ruigzheng2-c@my.cityu.edu.hk, hanfeng@cityu.edu.hk, xzhuang7@cityu.edu.hk)}
\thanks{M. Li is with the Key Laboratory of Intelligent Education Technology and Application of Zhejiang Province, Zhejiang Normal University, Jinhua, China. (Email:
mingli@zjnu.edu.cn)}
\thanks{$^\dagger$ Equal contribution}
\thanks{* Corresponding authors}}

\maketitle

%\def\thefootnote{*}\footnotetext{.}
%
%\def\thefootnote{$\dagger$}\footnotetext{}
%

%\author{IEEE Publication Technology,~\IEEEmembership{Staff,~IEEE,}
        % <-this % stops a space
%\thanks{This paper was produced by the IEEE Publication Technology Group. They are in Piscataway, NJ.}% <-this % stops a space
%\thanks{Manuscript received April 19, 2021; revised August 16, 2021.}}

%\markboth{Submitted to IEEE TNNLS SI: Graph Learning}%
%{Shell \MakeLowercase{\textit{et al.}}:}

%{Shell \MakeLowercase{\textit{et al.}}: Permutaion Equivariant Graph Framelets for
%	Heterophilous Semi-supervised Learning}



\maketitle

\begin{abstract}
The nature of heterophilous graphs is significantly different from that of homophilous graphs, which causes difficulties in early graph neural network models and suggests aggregations beyond the 1-hop neighborhood. In this paper, we develop a new way to implement multi-scale extraction via constructing Haar-type graph framelets with desired properties of permutation equivariance, efficiency, and sparsity, for deep learning tasks on graphs. We further design a graph framelet neural network model PEGFAN (Permutation Equivariant Graph Framelet Augmented Network) based on our constructed graph framelets. The experiments are conducted on a synthetic dataset and 9 benchmark datasets to compare performance with other state-of-the-art models. The result shows that our model can achieve the best performance on certain datasets of heterophilous graphs (including the majority of heterophilous datasets with relatively larger sizes and denser connections) and competitive performance on the remaining. 
\end{abstract}


\section{Introduction}

\IEEEPARstart{G}{raphs} are ubiquitous data structures for a variety of real-life entities, such as traffic networks, social networks, citation networks, chemo- and bio-informatics networks, etc. With the abstraction via graphs, many real-world problems that are related to networks and communities can be cast into a unified framework and solved by exploiting its underlying rich and deep mathematical theory as well as tremendously efficient computational techniques.  In recent years, graph neural networks (GNNs) for graph learning such as node classification \cite{GCN}, link prediction \cite{ying2018graph}, and graph classification \cite{zhang2018end}, have demonstrated their powerful learning ability and achieved remarkable performance \cite{survey,wu2020comprehensive,zhou2020graph,zhang2022deep,xia2021graph}.   In the particular field of node classification,  many GNN models follow the \textit{homophily} assumption, that is, the majority of edges connect nodes from the same classes (e.g., researchers in a citation network tend to cite each other from the same area), yet graphs with \textit{heterophily}, that is, the majority of edges connect nodes from different classes \cite{1},  do exist in many real-world scenarios. A typical example is in a cyber network where a phishing attacker usually sends fraudulent messages to a large population of normal users (victims) in order to obtain sensitive information. We refer to  \cite{2,3} for the limitations of early GNNs on homophilous graphs and a recent survey paper \cite{1}  on GNNs for heterophilous graphs.

Heterophilous graphs differ from homophilous graphs not only {\em spatially} in terms of distribution beyond the 1-hop neighborhood but also {\em spectrally} with larger oscillation in terms of the frequency distribution of graph signals under the graph Laplacian. Such properties bring challenges to learning on heterophilous graphs and demand new GNNs the ability to extract intrinsic information in order to achieve high performance. To enhance the influence of nodes from the same classes that are outside of 1-hop neighborhoods, one common approach is based on the  {\em multi-hop aggregation} to leverage information of $k$-hop neighborhoods,  $k\geq 2$. %see for instances \cite{MP, 7, 6}. 
Its effectiveness for heterophilous graphs is emphasized and theoretically verified in \cite{8}.
A common way to perform multi-hop aggregation is to utilize the powers of the adjacency matrix. Repeatedly applying Laplacian smoothing many times, prompted by using higher powers of adjacency matrix, can result in a convergence of vertex features within each connected component of the graph towards uninformative or identical values, a phenomenon referred to as over-smoothing \cite{2,rusch2023survey}. %However, the use of higher powers of adjacency matrix will result in a convergent phenomenon as shown in \cite{2} and may fail to capture meaningful information. 
Moreover, they may lead to dense matrices and cause computation and storage burdens. To seek further improvement, it is thus desirable to consider an alternative spatial resolution of graphs other than $k$-hop neighborhood. To answer this question, we work on the theory of wavelet/framelet systems on graphs which brings a notion of {\em scale} on graphs and wavelets/framelets corresponding to such scales. In this paper, we introduce and integrate a dedicated graph framelet system so as to perform {\em multi-scale extraction} on graphs.

Actually, classical wavelets/framelets in the Euclidean domains, e.g., see \cite{9,han2017framelets}, are well-known examples of multi-scale representation, which have been extended to irregular domains such as graphs and manifolds under similar principles in recent years, e.g., see \cite{10,11,dong2017sparse,li2022convolutional}. Some graph wavelets/framelets systems are also proposed and applied in GNNs for node and graph classifications \cite{12,14,13}. When the graph is reordered, it is natural to expect the produced wavelets/framelets to be reordered in the same way for robust learning. However, most of the graph wavelets/framelets do not possess such a property of  {\em permutation equivariance}.  That is, up to certain permutations, the constructed graph wavelet/framelet systems should be the ``same''  regardless of the underlying node orderings.  The work on {Haar-type} graph wavelets/framelets \cite{H2,11,H3,12, xiao2021adaptive} are ``piecewise-constant'' functions on graphs that depend on a given tree with certain underlying node ordering.  If new orderings are given, though the underlying graph and graph data are the same, the newly resulting graph wavelets/framelets are no longer the same. Without the property of permutation equivariance, the network outputs could vary with respect to graph reordering and thus lead to instability of the GNNs.

In this paper, we provide a novel and general method to construct Haar-type graph framelets having the permutation equivariance property, which further implies the permutation equivariance of our graph framelet neural network (GFNN) model  PEGFAN ({\em Permutation Equivariant Graph Framelet Augmented Network}). Our Haar-type graph framelets are constructed spatially with respect to a hierarchical structure on the underlying graph.
Scales in such systems correspond to the levels in the hierarchical structure in which higher levels are associated with larger groups of nodes. Multi-scale extractions via such graph framelets are regarded as alternatives and supplements for the usual multi-hop aggregations. Moreover, we show that our graph framelets possess sparse representation property, which leads to the sparsity property of the orthogonal projection matrix (framelet matrix) formed by stacking those framelet vectors at certain scales. This is in contrast to the high powers of adjacency matrices and their non-sparse nature. Furthermore, we apply our graph framelets in the neural network architecture design by using the framelet matrices at different scales as well as the adjacency matrices to form multi-channel input and perform multi-scale extraction through attention and concatenation. The state-of-the-art node classification accuracies on several benchmark datasets validate the effectiveness of our GFNN model.

In summary, the contribution of this paper is as follows: 1) We propose a novel and general method to construct Haar-type graph framelets that have properties of permutation equivariance, sparse representation, efficient computation, and so on.  2) We apply our Haar-type graph framelet system to extract multi-scale information and integrate it into a graph neural network architecture.  3) We demonstrate the effectiveness of our model for node classification on synthetic and benchmark datasets via extensive comparisons with several state-of-the-art GNN models.

\section{Related Work}

\textbf{Node Classification on Heterophilous Graphs.} Early work on node classification includes \cite{ChebyGCN,GCN, GRAPHSAGE}, which are some of the earliest examples of spectral and spatial GNNs.
\cite{GEOMGCN} (GEOM-GCN) is the first article that aims at heterophilous graphs.  \cite{zhao2022neighborhood}  proposes a topology augmentation graph convolutional network (TA-GCN) framework under the guidance of an NCC (neighborhood class consistency) metric.  \cite{chen2023exploiting}  proposes a GNN model called conv-agnostic GNN  (CAGNN) to enhance the performance of GNNs on the heterophily datasets by learning the neighbor effect for each node.   \cite{wu2023beyond}   proposes a  relation-based frequency adaptive GNN (RFA-GNN) that can adaptively pick up signals of different frequencies in each corresponding relation space in the message-passing process. \cite{8} identifies a set of key designs that boost learning under heterophily. \cite{WRGAT} proposes to convert the input graph into a computation graph with proximity and structural information so as to counter the limit imposed by node-level assortativity (homophily). \cite{GPRGNN} introduces a new generalized PageRank to jointly optimize node feature and topological information extraction, regardless of homophily or heterophily. \cite{diffwire} proposes two novel fully differentiable and inductive rewiring layers to mitigate the problems of over-smoothing, over-squashing, and under-reaching on both homophilous and heterophilous graphs. \cite{guo2023homophily} adopts a homophily-oriented deep heterogeneous graph rewiring method to increase the meta-paths subgraph homophily ratio, which helps improve the performance of heterogeneous graph neural network (HGNN) on heterophilous graphs. In \cite{spinelli2021fairdrop}, a random-edge dropping mechanism for increasing heterophily of graphs is proposed, aiming at enhancing fairness in GNNs' predictions. We refer to \cite{1} for a comprehensive review of graph neural networks for graphs with heterophily.

\textbf{Multi-hop Aggregation in GNNs.} Papers of \cite{7,6,8,5} are GNNs that adopt hidden layer concatenation and multi-hop aggregation and involve the powers of adjacency matrices. Thus, they resemble each other in terms of neural network architecture. The difference is that \cite{7, 6}  mainly deal with homophilous datasets. On the other hand, with emphasis on the heterophilous setting, \cite{8} theoretically shows the importance of concatenation of aggregation beyond the 1-hop neighborhood, with an addition on the importance of ego- and neighbor-embedding separation. Such a non-local neighborhood aggregation is also emphasized in \cite{GRAPHSAGE,7}.  The current state-of-the-art model FSGNN, i.e., Feature Selection GNN \cite{5},  is different from the previous ones by, in our interpretation, viewing the semi-supervised setting as a supervised setting in which multi-hop aggregation is regarded as input of different feature channels from different hops and were not applied in the following layers. As a result, its network architecture basically consists of a {\em mix-hop} \cite{6} layer and fully-connected layers with attention weights for different channels being applied before the concatenation.  It is worth mentioning that a recent work \cite{LINKX}  on large-scale heterophilous node classification is very similar to \cite{5}, in which input channels were limited to the $0$-hop and the $1$-hop.

\textbf{Graph Wavelets/Framelets.} Papers of \cite{H1,H2, 10, 11,dong2017sparse,H3,12, xiao2021adaptive} are work of graph wavelets/framelets in which \cite{10,dong2017sparse} are spectral-type and the rest are Haar-type. A framelet system differs from the classical  (orthogonal) wavelet system by being a frame in a Hilbert space and offering redundant representation. The Haar-type wavelet system in \cite{H1} is defined for different nodes as centers. \cite{H2,11,H3,12,xiao2021adaptive} define graph wavelets/framelets under a given tree and they differ in the interpretation and generation of the tree. \cite{H2} applies to trees from graphs. In \cite{11}, the tree is represented as a filtration on $[0,1]$, and the wavelet system is equivalent to an orthogonal basis of tree polynomials. Similar to \cite{11},   the trees are further generalized to hierarchical partitions of $[0,1]^2$ in \cite{H3} and apply to directed graphs, and a Haar-type wavelet system for directed graphs is thus constructed. \cite{xiao2021adaptive} further generalizes the work in \cite{H3, han2019directional} by considering the constructions of Haar-type framelet systems on any compact set in $\mathbb{R}^d$ under a given hierarchical partition and adapt the construction of directed graph framelets to such cases. %\cite{li2022convolutional} gives more general sufficient conditions for constructing such Haar-type framelet systems on compact sets and applies to the construction of spherical neural networks for spherical signal processing. %\cite{li2022convolutional} gives more general sufficient conditions for constructing such Haar-type framelet systems on  compact sets and applies to the construction of spherical neural networks for spherical signal processing.

\begin{figure}[htpb]
	\centering
	\subfloat[Graph $\g_1$]{\includegraphics[scale=0.3]{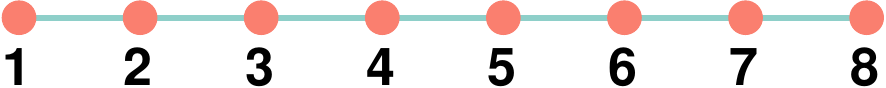}}\\
	\subfloat[Hierarchical partition $\p_4$]{\includegraphics[scale=0.3]{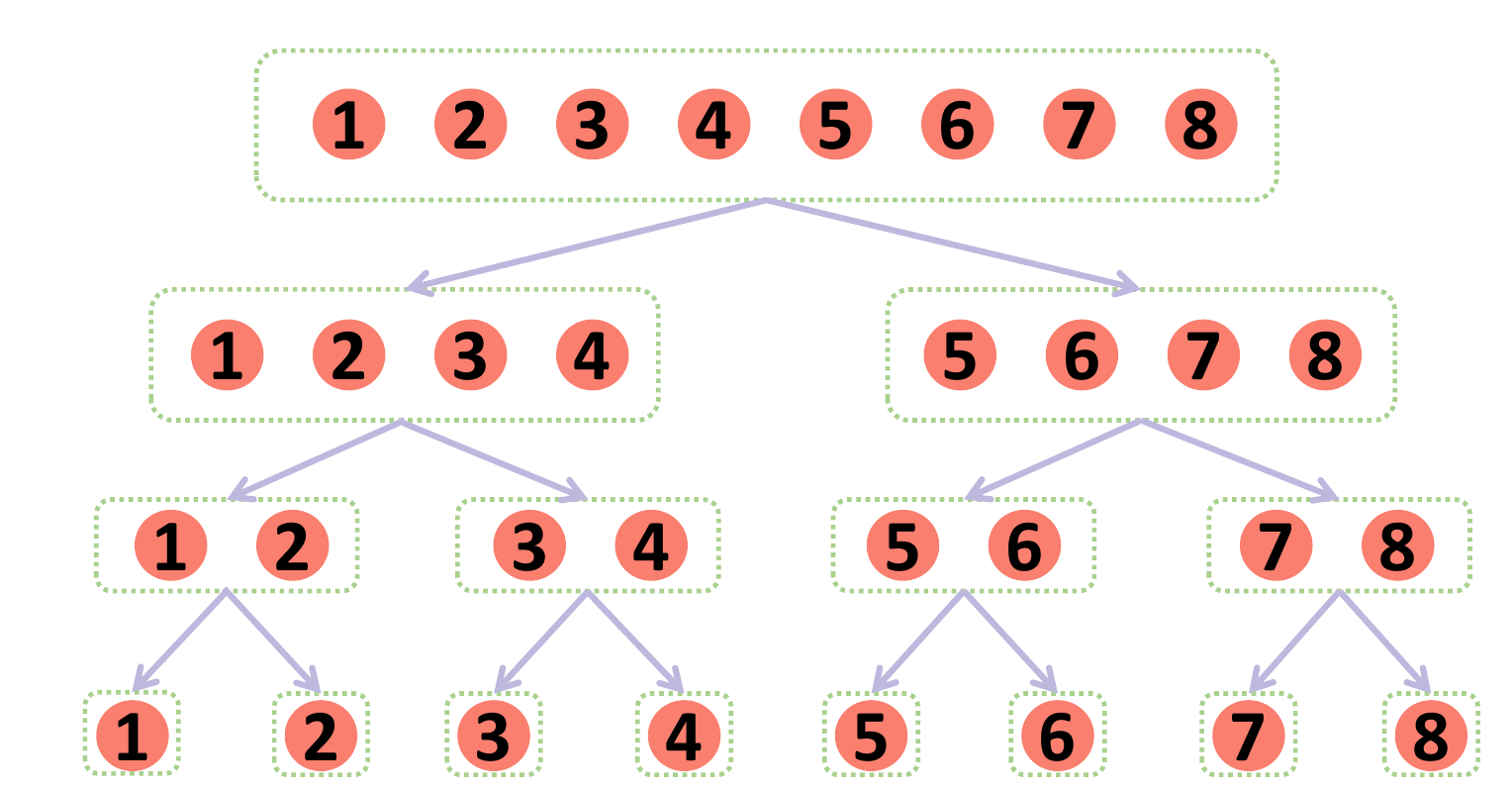}}\\
     \subfloat[Graph framelets]{\includegraphics[scale=0.3]{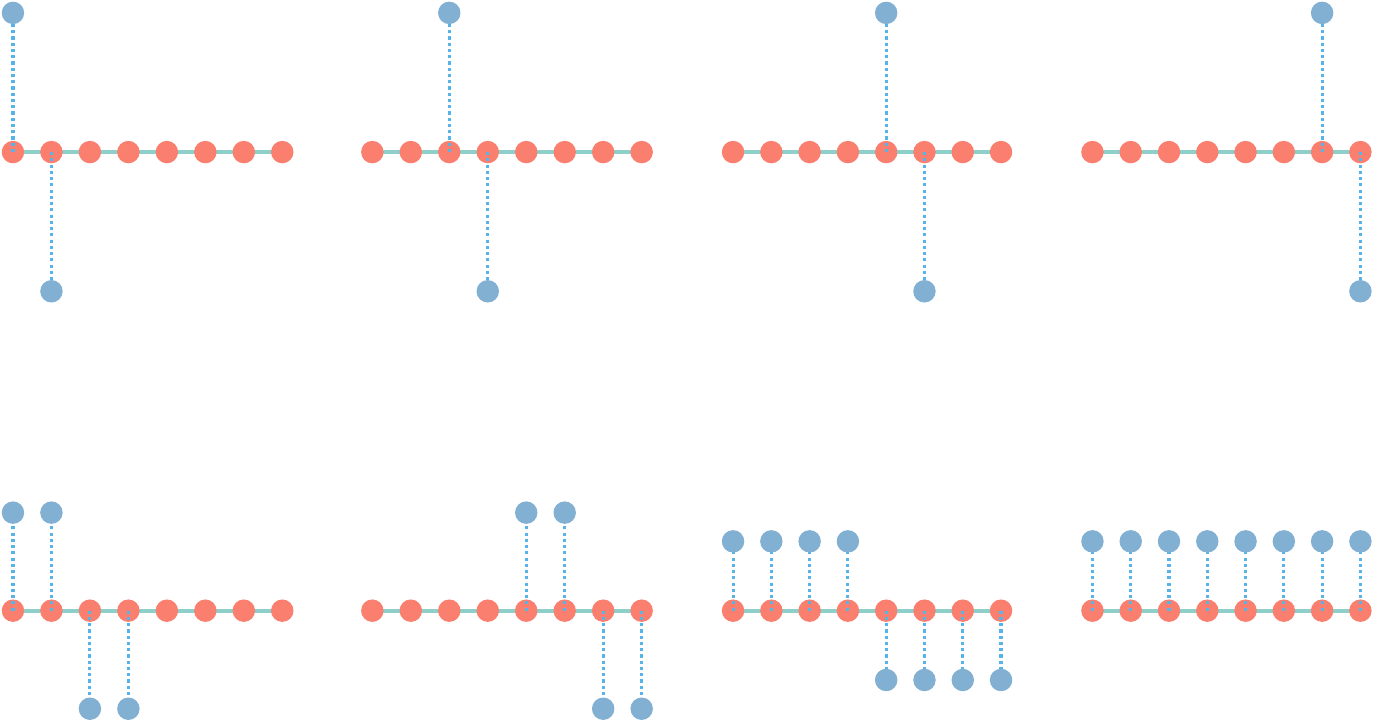}}
	    \caption{Graph framelets (bottom) w.r.t. $\g_1$ (top) and the hierarchical partition $\p_4
		$ (middle). The blue points are the values of graph framelets on each node. The height of blue points represents the value of graph framelets $\psi_{\bm \Lambda}$ (bottom 8). The blue points are above the corresponding nodes if values are positive, and below otherwise.}
	\label{fig:short-a}
\end{figure}

\section{Permutation Equivariant Graph Framelets}\label{GHF}

In this section, we develop Haar-type graph framelet systems, the binary Haar graph framelets, with properties of tightness, sparsity, efficiency, and permutation equivariance, which yield robustness and effective algorithms for the GFNN model PEGFAN to be introduced in \cref{sec:model}. All proofs of the main results in this paper are postponed to \cref{app:A}.

\subsection{Preliminaries}
Let $\g = (\v,\w)$ be a graph, where $\v = \{v_1,\dots,v_n\}$ is the vertex set containing $n$ vertices (or equivalently, we simply identify $\v = \{1,2,\dots,n \}$), and $\w \subset \v \times \v$ is the edge set of ordered pairs $(i,j)$.   The adjacency matrix $\b A : \v \times \v \rightarrow \RR$ of $\g$ is a matrix of size $n\times n$ such that its $(i,j)$-entry $a_{ij}$ is the weight on edge $(i,j)$ and $a_{ij}=0$ if $(i,j) \notin \w$. We consider only undirected graphs in this paper, i.e., $\b A^\top  = \b A$. We denote $\tilde{\b{A}}:= \b{D}^{-1/2}\b{A}\b{D}^{-1/2}$ with $\b D$ being the diagonal degree matrix of $\g$, whose diagonal elements defined as $d_{ii}= \sum_{j=1}^n a_{ij}$. A signal $\b f = [f_1,\dots,f_n]^\top$ on the graph is defined as $\b f : \v \rightarrow \RR$ with $\ell_2$ norm  $\|\b f\|^2 = \sum_{i=1}^{n}|f_i|^2 < \infty$. All such $\ell_2$ signals on $\g$ form a Hilbert space $L_2(\g)$ under the usual inner product. A collection $\{\b e_m: m \in[M]  \}\subset L_2(\g)$ is a {\em tight frame} of $L_2(\g)$ if
 $\b f = \sum_{m=1}^M \langle \b f,\b e_m \rangle \b e_m$ for all  $\b f \in L_2(\g)$, where $\langle\cdot,\cdot\rangle$ is the inner product and we denote $[M]:=\{1,\ldots,M\}$. We denote the $i$-th column vector and row vector of a matrix $\b M$, by $\b M_{:i}$ and $\b M_{i:}$, respectively.

For $K\ge 2$, we call a  sequence $\p_K: = \{\v_j: j = 1,\dots,K \}$   of sets as a {\em $K$-hierarchical clustering} of $\v$ if each $\v_j:=\{s_{\b\Lambda} \subset \v:  \dim(\b\Lambda)=j\}$ is a partition of $\v$, i.e., $\v=\cup_{\b\Lambda} s_{\b\Lambda}$, and $ \v_j$ is a refinement of $ \v_{j-1}$, where we use the index vector $ {\b{\b\Lambda}}=(\lambda_1,\ldots,\lambda_j)\in\NN^j$ to encode position, level $j$, and parent-children relationship, of the clusters $s_{\b\Lambda}$. \cref{fig:short-a} gives an example of a $K$-hierarchical clustering. Let us denote $\v_1 = \{s_{(1)}=\v=\{1,2,\dots,8\} \}$. Then according to parent-children relationship in \cref{fig:short-a}, we have $\v_2 = \{ s_{(1,1)}, s_{(1,2)} \}$ where $s_{(1,1)}=\{1,\dots,4\}$ and $s_{(1,2)}=\{5,\dots,8\}$. Similarly, we have $\v_3 = \{ s_{(1,1,1)},s_{(1,1,2)},s_{(1,2,1)},s_{(1,2,2)} \}$ where $s_{(1,1,1)} = \{1,2\}$, $s_{(1,1,2)} = \{3,4\}$, $s_{(1,2,1)} = \{5,6\}$ and $s_{(1,2,2)} = \{7,8\}$, and thus $s_{(1,1,1,1)} = \{1  \} $ and $s_{(1,1,1,2)} = \{2  \} $.
Here $\dim(\bm\Lambda)$ denotes the length of the index vector. If $ s_{  {\b\Lambda}}\in \v_{j}$ is a parent, then the index vectors of its children are appended with an integer, i.e. $(\b\Lambda,i)$, indicating its $i$-th child, and thus the child is denoted by $ s_{ ( {\b\Lambda},i)}\in \v_{j+1}$. Then we have the parent-children relationship $s_{ ( {\b\Lambda}, i)} \subset s_{  {\b\Lambda}}$. We denote the number of children of $s_{  {\b\Lambda}}$ by $L_{  {\b\Lambda}}$. Unless specified, we consider $K$-hierarchical clustering $\p_K$ with  $\v_K = \{\{1\},\ldots,\{n\}\}$ and  $\v_1=\{[n]\}$ being a singleton, i.e., $\p_K$ is a {\em tree}. 

In classical wavelet/framelet theory, an important concept is the multiresolution analysis (MRA). One of the most important ideas is to find a sequence of subspaces $\{ V_j \} \subset L_2(\RR)$ such that $V_j \subset V_{j+1}$ and $\cup_{j\in\ZZ} V_{j} = L_2(\RR)$. If there exists $\phi(t) \in V_0$ such that $\{ \phi(t-b)  \}_{b\in\ZZ}  $ forms an orthonormal basis of $V_0$ and $f(t) \in V_j$ if and only if $f(2t)\in V_{j+1}$, then we can find a mother wavelet $\psi(t)$ such that $\{2^{j/2}\psi({2^jt-b})\}_{j,b\in \ZZ} $ forms an orthonormal basis for  $L_2(\RR)$. For example, let $\phi(t) = \chi_{[0,1)}(t)$ and $\psi(t) = \phi(2t)-\phi(2t-1)$, then the resulting wavelet is the so-called Haar wavelet. However, the translation and dilation operators are not naturally defined for graph signals. Fortunately, if we look at the support of Haar wavelets, we can find that the union of the support of elements in $\mathfrak{W}_j:=\{2^{j/2}\psi({2^jt-b})\}_{b\in\NN}$ is equal to $\RR$ for a fixed $j$ and the collection of the support of elements in $\mathfrak{W}_{j+1}$ is a refinement of that of $\mathfrak{W}_{j+1}$. Hence, these supports actually form a hierarchical partition. Based on this observation, a natural way to define translations on the graph is to generalize hierarchical partitions to the graph, see \cite{xiao2021adaptive, li2022convolutional}. For example, when mapping each node in a graph to an interval on $[0,1]$, then based on the hierarchical partition on $[0,1]$, graph framelets can be constructed similarly as classical Haar wavelet \cite{xiao2021adaptive}. Below, we provide general conditions in \cref{thm1}  for constructing Haar graph framelets based on a $K$-hierarchical clustering.

\subsection{Main Construction}\label{main}

%Our graph framelet is inspired by the spherical Haar framelet \cite{li2022convolutional}, which is well-localized and thus sparse. The similarity structure on the graph is considered during the framelet construction.

Given $\p_K$, we define the unit scaling vectors  $\b\phi_{  {\b\Lambda}}$ (similar to scaling functions  for  $V_j$'s  on a MRA) iteratively from $\dim(\b\Lambda)=K$ to $\dim(\b \Lambda)=1$. When $\dim ( {\b\Lambda}) = K$, each cluster (node)  $s_{  {\b\Lambda}}$  contains only one vertex in graph $\g$ (see \cref{fig:short-a}), thus we define $\b\phi_{  {\b\Lambda}} =\b I_{:i}$, where $i\in s_{  {\b\Lambda}} \subset \v$ and $\b I_{:i}$ is the $i$-th column of the identity matrix $\b I \in \RR^{n \times n}$.
When $\dim ( {\b\Lambda}) < K$, we define
\begin{align}\label{1}
\b\phi_{  {\b\Lambda}} := \sum_{\ell\in[L_{  {\b\Lambda}}]} p_{(  {\b\Lambda}, \ell)}  \b\phi_{(  {\b\Lambda}, \ell)},
\end{align}
%\[ \b\phi_{  {\b\Lambda}} := \sum_{\ell =1 }^{L_{  {\b\Lambda}}} p_{(  {\b\Lambda}, \ell)}  \b\phi_{(  {\b\Lambda}, \ell)},\]
where $\b p_{  {\b\Lambda}} = [p_{(  {\b\Lambda}, 1)}, \dots, p_{(  {\b\Lambda}, L_{ {{\b\Lambda}}})}]^\top  \in \RR^{L_{  {\b\Lambda}}}$ and $\| \b p_{  {\b\Lambda}} \|=1$. Obviously, $\b\phi_{  {\b\Lambda}}$ is with support $\supp \b\phi_{  {\b\Lambda}} = s_{  {\b\Lambda}}$ and $\|\b\phi_{  {\b\Lambda}} \|=1$.
For framelet vectors on the graph, we define $\b\psi_{(  {\b\Lambda}, m)}$, $m \in[ M_{{\b\Lambda}}]$ for some $M_{\b\Lambda}\in\NN$ by    %\[\b\psi_{(  {\b\Lambda}, m)} := \sum_{\ell =1 }^{L_{  {\b\Lambda}}} \left(\b B_{ {\b\Lambda} }\right)_{m\ell}  \b\phi_{(  {\b\Lambda}, \ell)},\]
\begin{align}\label{2}
\b\psi_{(  {\b\Lambda}, m)} := \sum_{\ell\in[L_{  {\b\Lambda}}]} \left(\b B_{ {\b\Lambda} }\right)_{m,\ell}  \b\phi_{(  {\b\Lambda}, \ell)},
\end{align}
from some matrices $\b B_{  {\b\Lambda}}\in \RR^{M_{ {{\b\Lambda}}} \times L_{  {\b\Lambda}} }$. \cref{thm1} characterizes  when $\b\phi_{  {\b\Lambda}}$ and $\b\psi_{( {\b\Lambda}, m)}$ form a tight frame of $L_2(\g)$.

\begin{theorem}[General characterization]\label{thm1} Let $\p_K$ be a $K$-hierarchical clustering on a graph $\g$. Then the matrices $\b B_{ {{\b\Lambda}}}$ and vectors $\b p_{ {{\b\Lambda}}}$ satisfy $\b B_{ {{\b\Lambda}}} \b B_{ {{\b\Lambda}}}^\top \b B_{ {{\b\Lambda}}} = \b B_{ {{\b\Lambda}}}$, $ \b B_{ {{\b\Lambda}}} \b p_{ {{\b\Lambda}}} = \b 0 $, and $\rank(\b B_{ {{\b\Lambda}}}) = L_{ {{\b\Lambda}}}-1$ for all $\b\Lambda$ with $\dim( { {\b\Lambda}}) = j_0,\ldots,K$ if and only if for any $j_0\in[K]$,
 the collection
	$
	\mathcal{F}_{j_0}(\p_K):=\{ \b\phi_{  {\b\Lambda}} : \dim ( {\b\Lambda}) = j_0 \} \cup \{ \b\psi_{  {\b\Lambda}} : \dim ( {\b\Lambda}) = j\}_{j= j_0+1}^K
	$ defined by \cref{1,2}
	is a tight frame of $L_2(\g)$.
\end{theorem}

We remark that \cref{thm1} provides a more general sufficient and necessary condition than that in \cite{li2022convolutional}, for all graph framelets having the form \eqref{1} and \eqref{2} to be a tight frame. When we use the  Haar graph framelets to extract frequency features of graph signals, general graph wavelets/framelets (\eqref{1} and \eqref{2}) can be viewed as multi-scale representation systems in which the notion of `scale' is different from the usual $k$-hop neighborhood in graphs and serve as an alternative to capture long-range information. 

Besides, the given $\p_K$ in the proposed construction is not specified. The advantage of the generality of this definition is that there is no constraint on how the $\p_K$ is generated: one can use solely the edges or combine the edges and node features to generate $\p_K$, etc. Thus this provides great potential in theories and applications. As well shown in experiments, clustering graph nodes based only on adjacency matrices is capable of providing nice graph framelets that help improve the learning abilities of neural networks on node classification tasks.

The following example shows a close relationship between our framelet systems and the traditional Haar graph basis.

\begin{exam}[Path graph and Haar basis]\label{g1}
	Given a path graph $\g_1$ with $8$ nodes $\v = \{1,2,\dots,8 \}$. If we choose hierarchical clustering $\p_4=\{\v_1,\v_2,\v_3,\v_4\}$ with
	$\v_1=\{ s_{(1)} \}$, $\v_2=\{
	s_{(1,1)}, s_{(1,2)} \}$, $
	 \v_3=\{
	s_{(1,1,1)}, s_{(1,1,2)},s_{(1,2,1)}, s_{(1,2,2)} \}$, $\v_4=\{\ab 
	s_{(1,1,1,1)}, \ab s_{(1,1,1,2)}, \ab s_{(1,1,2,1)}, \ab s_{(1,1,2,2)}, 
	\ab s_{(1,2,1,1)}, \ab s_{(1,2,1,2)}, \ab s_{(1,2,2,1)}, \ab s_{(1,2,2,2)} \}$,
%	\begin{align*}
%	&\v_1=\{ s_{(1)} \}, \quad \v_2=\{
%	s_{(1,1)}, s_{(1,2)} \},\\
%	& \v_3=\{
%	s_{(1,1,1)}, s_{(1,1,2)},s_{(1,2,1)}, s_{(1,2,2)} \} ,\\
%	&\v_4=\{
%	s_{(1,1,1,1)}, s_{(1,1,1,2)},s_{(1,1,2,1)}, s_{(1,1,2,2)},\\
%	&\quad\quad\quad s_{(1,2,1,1)},s_{(1,2,1,2)},s_{(1,2,2,1)},s_{(1,2,2,2)} \},
%	\end{align*}
	and $\b p_{ {{\b\Lambda}}} = [\f{1}{\sqrt{2}},\f{1}{\sqrt{2}}]^\top$ and $\b B_{ {{\b\Lambda}}} = [\f{1}{\sqrt{2}},-\f{1}{\sqrt{2}}]$ for all $ {{\b\Lambda}}$ (note that each parent has exactly two children $L_{ {{\b\Lambda}}} = 2$), then the graph framelet system $\mathcal{F}_{j_0}(\p_K)$ with $K=4$ as in \cref{thm1}  is a Haar basis  for any $j_0\in[K]$. See \cref{fig:short-a} for illustration.
\end{exam}

Based on the general conditions in \cref{thm1}, we further investigate the specific structure of $\b B_{\b\Lambda}$.  We give the following proposition that completely characterizes the structure of matrices $\b B_{\b\Lambda}$ in \cref{thm1}.

\begin{proposition}\label{B}
Let $\b p$ be a unit vector of length $L\ge1$, that is, $\|\b p\|=1$. 
Assume that $\b B\in  \RR^{M \times L}$ with $M\ge L-1$ is a matrix such that $\b B \b  p = 0$ and $\rank(\b B) = L-1$. Then  $\b B \b B^\top \b B = c \b B$ for some constant $c$ if and only if $\b B^\top \b B = c (\b I-\b p \b p^\top)$. In particular, if $c\neq 0$, then $\b P := [\b p, \f{1}{\sqrt{c}}\b B^\top]$ satisfies $\b P\b P^\top = \b I$. 
\end{proposition}

\cref{B} shows that $\b B_{\b\Lambda}$ is from the (matrix) splitting of a rank $L-1$ matrix $\b I - \b p_{{\b\Lambda}}\b p_{\b\Lambda}^\top$. Notice that the role of elements in $\b p_{\b\Lambda}$ in \cref{1} is to give weights to each cluster $s_{(\b \Lambda, \ell)}$. One typical scenario is that each child cluster is of equal importance, which means that the vector $\b p_{\b \Lambda}$ is a vector with all equal elements. On the other hand, it could be too involved to use matrix splitting techniques \cite{han2010matrix,zhuang2012matrix,mo2012matrix,han2013algorithms} for obtaining the matrix  $\b B_{\b\Lambda}$. We next show that we can obtain matrices $\b B_{\b\Lambda}$ by simply permuting a fixed vector $\b w$ such that each of its elements appears with equal chance. Under this hypothesis of equal importance and equal chance, in the following result, we introduce a binary Haar graph framelet system by a careful design of the matrices $\b B_{ {{\b\Lambda}}}$ and $\b p_{ {{\b\Lambda}}}$. The word \emph{binary} here is chosen since each nonzero coefficient of high-frequency framelets in \cref{2} only takes from $\{1,-1\}$ (without normalization). We show that such graph framelet systems $\mathcal{F}_{j_0}(\p_K)$ have many desirable properties including permutation equivariance.

For each pair $(\ell_1,\ell_2)$ with $1\le \ell_1<\ell_2\le L_{\b\Lambda}$, define a vector $\b w_{ {{\b\Lambda}}}^m$ of size $L_{\b \Lambda}\times 1$ by
\begin{align}\label{Bw:w}
(\b w_{\b\Lambda}^m)_{\tau}=
\begin{cases}
\frac{1}{\sqrt{L_{\b\Lambda}}} & \tau = \ell_1;\\
\frac{-1}{\sqrt{L_{\b\Lambda}}} & \tau= \ell_2;\\
0 & \mbox{otherwise},
\end{cases}	
\end{align}		
where  $m:=m(\ell_1,\ell_2,L_{\b\Lambda}):=\frac{(2L_{\b\Lambda}-\ell_1)(\ell_1-1)}{2}+\ell_2-\ell_1$  is ranging from $1$ to $M_{\b\Lambda}:=\frac{L_{\b \Lambda}(L_{\b \Lambda}-1)}{2}$ for all possible pairs $(\ell_1,\ell_2)$ with $1\le \ell_1<\ell_2\le L_{\b\Lambda}$. Note that  $\b w_{ {{\b\Lambda}}}^m$ has only two non-zero entries locating at the $\ell_1$-th and $\ell_2$-th position, respectively. Such a $\b w_{\b\Lambda}^m$ will be used as the $m$-th row of the matrix $\b B_{\b\Lambda}$.

\begin{corollary}[Binary Haar graph framelets]\label{cor}
Let $\p_K$ be a $K$-hierarchical clustering on a graph $\g$.
	Let  $\b p_{ {{\b\Lambda}}} = \f{1}{\sqrt{L_{ {{\b\Lambda}}}}}\bm{1}$ be a constant vector of size $L_{\b \Lambda}\times 1$ and $\b B_{\b\Lambda}:=[\b w_{\b\Lambda}^1,\ldots,\b w_{\b\Lambda}^{M_{\b\Lambda}}]^\top$ with $\b w_{ {{\b\Lambda}}}^m$ being given as in \cref{Bw:w}. Define $\mathcal{F}_{j_0}(\p_K)$ as in \cref{thm1}. Then $\mathcal{F}_{j_0}(\p_K)$ is a tight frame for $L_2(\g)$ for any $j_0\in[K]$.
\end{corollary}

% \begin{exam}\label{g2}
% 	Given a graph $\g_2$ with $7$ nodes $\v = \{1,2,\dots,7 \}$. We choose hierarchical clustering $\p_3=\{\v_1,\v_2,\v_3\}$ with $\v_1=\{ s_{(1)} \}$,
% 	$\v_2=\{
% 	s_{(1,1)}, s_{(1,2)} \}$,
% 	 $\v_3=\{
% 	s_{(1,1,1)}, s_{(1,1,2)},s_{(1,2,1)}, s_{(1,2,2)} \} \}$,
% %	\begin{align*}
% %	&\v_1=\{ s_{(1)} \},\quad
% %	\v_2=\{
% %	s_{(1,1)}, s_{(1,2)} \},\\
% %	& \v_3=\{
% %	s_{(1,1,1)}, s_{(1,1,2)},s_{(1,2,1)}, s_{(1,2,2)} \} \},
% %	\end{align*}
% 	and $\b p_{ {{\b\Lambda}}} $ and $\b B_{ {{\b\Lambda}}} $  described as in \cref{cor}.  Note  that $L_{(1)}=2$, $L_{(1,1)}=3$ and $L_{(1,2)}=4$.  The  system $\mathcal{F}_{j_0}(\p_K)$ with $K=3$ as in \cref{cor} is a tight frame for any $j_0\in[K]$. \cref{fig:short-b} shows that our framelet is well-localized with multi-scale information.
% \end{exam}

\begin{remark}\label{rm2}
In fact, the framelets obtained in Examples~\ref{g1} %and \ref{g2}  
belongs to binary  Haar graph framelets (See \cref{fig:short-a}  for illustration). The matrix $\b B_{\b\Lambda}$ is formed by permuting $1,-1$ of the specific type of vectors $\b w = [1,-1,0,\ldots,0]$ to all possible positions. In fact, more general types of vectors $\b w$ can be served to form the matrix $\b B_{\b\Lambda}$ through permutations. 
%More details of Examples~\ref{g1} and \ref{g2} for \cref{cor} can be found in \cref{app:examples} of the supplemental file.
%More general construction of $\b B_{\b\Lambda}$ is given in \cref{w} in \cref{app:A}.
\end{remark}

We next focus on the sparsity, efficiency, and permutation equivariance of the binary Haar graph framelets constructed in \cref{cor}.
\subsection{Sparsity}

Notice that if each row of the matrix $\b B_{ {{\b\Lambda}}}$ is sparse, then the produced $\psi_{(  {\b\Lambda}, m)}$ is also sparse. For the binary Haar graph framelets,  each row of $\b B_{ {{\b\Lambda}}}$  only contains two nonzero values. {Let $L_j := \max_{\dim( {{\b\Lambda}}) = j} L_{ {{\b\Lambda}}}$, then it is easy to see that the number $\left \| \b \psi_{(  {\b\Lambda}, m)} \right \|_0$ of non-zero entries of $\b \psi_{(  {\b\Lambda}, m)} $ satisfies $\left \| \b \psi_{(  {\b\Lambda}, m)} \right \|_0 \leq 2L_{j+1}$, for all $\dim( {{\b\Lambda}})=j$. When the hierarchical clustering  is balanced and $\dim( {{\b\Lambda}})$ is large, high-level framelets $\b \psi_{(  {\b\Lambda}, m)}$ are well-localized and thus sparse.

Besides the sparsity of the framelets, we also want to know when the framelet coefficients of a signal are sparse, which is the desired property of sparse representation of framelets. Different coefficients represent different scales. The sparse representation property plays an important role in feature extraction and representation for classification tasks. In node classification, piecewise constant signals are of great importance, e.g., in one-hot label encoding \cite{brownlee2017one}. Hence, it is valuable to study the framelet coefficients of the piecewise constant signals. Let $\mathcal{F}_{j_0}(\p_K):=\{ \b\phi_{  {\b\Lambda}}, \dim ( {\b\Lambda}) = j_0 \} \cup \{ \b\psi_{  {\b\Lambda}}, \dim ( {\b\Lambda}) = j  \}_{j=j_0+1}^K=:\{\b u_i\}_{i=1}^{M_{\g}}$ be a binary Haar graph framelet system with $M_{\g}$ elements  and define $\hat{\b f} \in \RR^{M_{\g}}$ to be the {\em framelet coefficient vector} with its $i$-th element $(\hat{\b f})_i := \langle \b f, \b u_i \rangle$ for a signal $\b f$. In what follows, we  denote $\b F:=[\b u_1,\ldots,\b u_{M_{\g}}]$ to be the matrix representation of the graph framelet system $\mathcal{F}_{j_0}(\p_K)$. Then, $\hat{\b f} = \b F^\top \b f$. We have the following result regarding the sparsity of $\hat{\b f}$.

\begin{theorem}[Binary Haar graph framelet transform preserving sparsity]\label{sparsity} Let $\mathcal{F}_{j_0}(\p_K)$ be a binary Haar graph framelet system defined as in \cref{cor}. Assume that $\max_{\dim( {{\b\Lambda}}) >0}L_{ {{\b\Lambda}}} \leq h$.  Then for a signal $\b f \in \RR^n$, the framelet coefficient vector $\hat {\b f}$ satisfies $\|\hat{\b f} \|_0 \leq(K-1)(h-1) \|\b f\|_0$.
\end{theorem}

\begin{remark}\label{rem3}
If for all $ {{\b\Lambda}}$, we have $L_{ {{\b\Lambda}}}=h$ for some integer $h\ge2$, then $K = O( \log_h n)$ and hence $\|\hat{\b f} \|_0 = \|\b f\|_0 \cdot O( h \log_h n )$, which shows that our binary  Haar graph transform preserves sparsity for sparse signals. In fact,  the total number $M_\g$ of elements in $\mathcal{F}_{j_0}(\p_K)$ in this case is of order $O(nh)$. When $\|f\|_0\ll n$, we see that $\|\hat {\b f}\|_0\ll   O(nh)=M_\g$.   \cref{sparsity} can be extended to other type of matrices $\b B_{ {{\b\Lambda}}}$ that is row-wise sparse.

\end{remark}

% \begin{figure}[t]
% 	\centering
% 	%\fbox{\rule{0pt}{2in} \rule{0.9\linewidth}{0pt}}
% 	\includegraphics[width=1.0\linewidth]{figures/examples3.pdf}
	
% 	\caption{Partition permutation. The height of blue points represents the value of graph framelets, and they are above the graph if values are positive, and below otherwise. The red box indicates the permuted subtree. The bottom part shows a graph framelet and its variant after partition permutation.}
% 	\label{fig:pp-short}
% \end{figure}

\begin{figure}[t]
	\centering
	\subfloat[$s_{(1,1,1)}$ of $\p_4$]{\includegraphics[width=3cm]{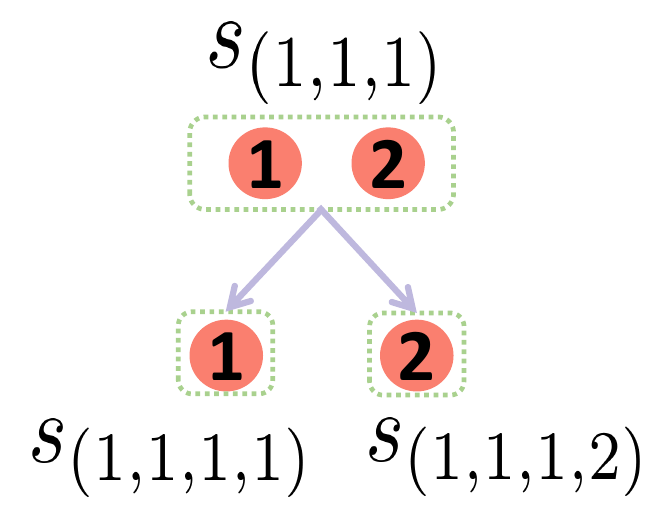}}\hspace{12pt}
	\subfloat[Permute children of $s_{(1,1,1)}$]{\includegraphics[width=3cm]{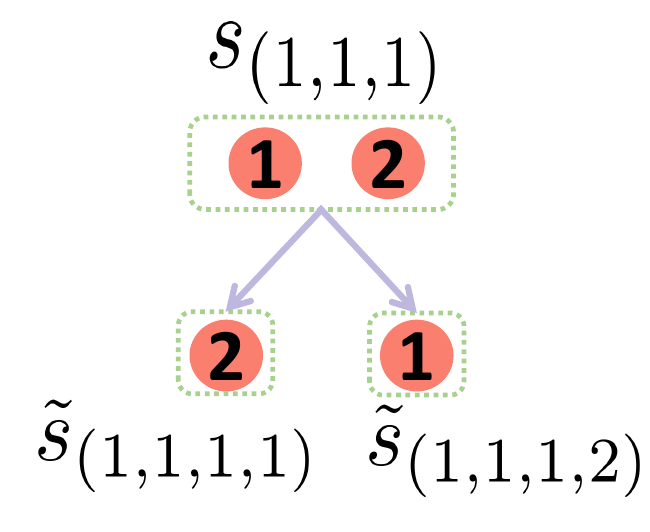}}\\
     \subfloat[$\b\psi_{(1,1,1)}$]{\includegraphics[width=4cm]{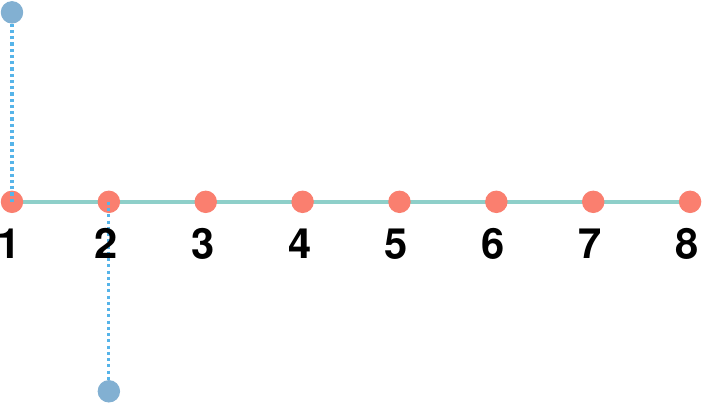}}\hspace{12pt}
	\subfloat[$\tilde{\b\psi}_{(1,1,1)}$]{\includegraphics[width=4cm]{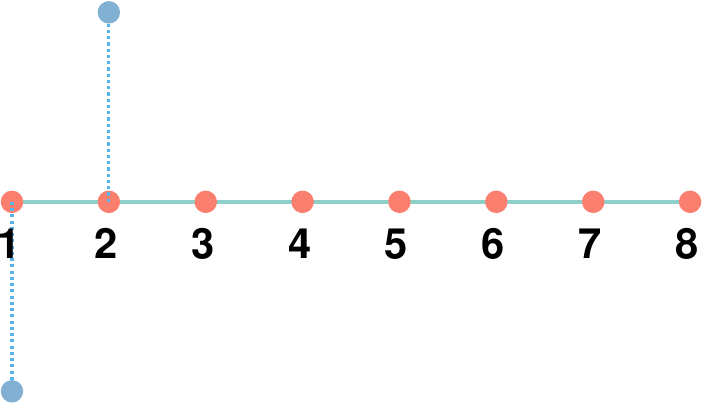}}
	\caption{Partition permutation. Consider the graph $\g_1$ and $\p_4$ in \cref{fig:short-a}. Let us permute the order of children of $s_{(1,1,1)}=\{1,2\}$ in $\p_4$ only while keeping other part of $\p_4$ and $\b p_{\b \Lambda}$ and $\b B_{\b\Lambda}$. The orginial framelet is given by $\b\psi_{(1,1,1)}$ (w.r.t. $s_{(1,1,1)}$) and the new framelet is given by $\tilde{\b\psi}_{(1,1,1)}$ (w.r.t. $\tilde s_{(1,1,1)}$). }
	\label{fig:pp-short}
\end{figure}

\subsection{Efficiency}
%
%\begin{figure}[htpb!]
%	\centering
%	%\fbox{\rule{0pt}{2in} \rule{0.9\linewidth}{0pt}}
%	\includegraphics[width=0.5\linewidth]{figures/examples4.pdf}
%	
%	\caption{Node permutation.}
%	\label{fig:lp}
%\end{figure}

Graph Fourier basis based on graph Laplacian is of great importance in graph neural networks. However, the computational complexity and space complexity of generating graph Fourier basis could be as large as $O(n^3)$ and $O(n^2)$,  respectively. Hence, these reasons prevent it from being more flexible in practice when $n$ is large and the graph Laplacian is not sparse. On the other hand, when using our binary Haar graph framelets, we have an efficient way to compute our framelets as well as the framelet coefficient vector via sparse computation. For the rest of this paper, when we discuss computational complexity, we assume that all matrix/vector operations are done by using sparse operations, i.e., the operations are evaluated only on non-zero entries.

	\begin{theorem}\label{fast1}	
		Let $h>1$ be a positive integer. Assume that the $K$-hierarchical clustering $\p_K$ satisfies  $n=O(h^{K-1})$ with $h := \max_{\dim( {{\b\Lambda}}) >0}L_{ {{\b\Lambda}}}$. For $j_0\in[K]$, let $\b F=[\b u_1,\ldots,\b u_{M_\g}]$ be the framelet matrix with respect to the binary Haar graph framelet system $\mathcal{F}_{j_0}(\p_K)$ as given in \cref{cor}.
Then, for all $j_0\in[K]$, 		
the number $M_\g$ of framelet vectors in $\mathcal{F}_{j_0}(\p_K)$  is of order  $O(nh)$,  the computational complexity of generating all $\b u_m$, $m=1,\ldots, M_\g$, in  $\b F$ is of order $O(nh\log_h n)$, and 
the total number $\mathrm{nnz}(\b F)$ of nonzero entries in $\b F$ is of order $O(nh\log_h n)$.	
\end{theorem}

\begin{remark} %We postpone the  proof of \cref{fast1} to  \cref{app:A}.
In practice, $h$ is usually small (e.g., 2, 4, or 8), and hence $\b F$ is a sparse matrix. \cref{fast1} shows that our binary Haar graph framelet systems are efficient in processing datasets with large graphs. Moreover, the framelet coefficient vector $\hat{\b f}$ can be computed with the computational complexity of order $O(nh)$ as well. 
See \cref{fast_alg} in \cref{app:alg}  for the fast decomposition and reconstruction algorithms using our graph framelet systems. 
\end{remark}

%\begin{theorem}\label{fast}	
% Assume that the $k$-hierarchical clustering $\p_k$ satisfies  $n=O(h^{k-1})$ with $h := \max_{\dim( {{\b\Lambda}}) >0}L_{ {{\b\Lambda}}}$. Then the framelet coefficient vector with respect to the binary graph Haar framelet system $\mathcal{F}_{j_0}(\p_k)$ can be computed in the order $O(n\ln n)$ of complexity.
%\end{theorem}

%Instead of directly computing framelets, many deep learning algorithms employ fast wavelet transform and inverse wavelet transform to process a signal or an image. Our framelet also has fast decomposition and reconstruction algorithms (not only for binary graph Haar framelet but also for general matrices $\b B_{ {\b\Lambda}}$), which is postponed to the supplementary material. When in $\p_k$, each parent has the same number of children, it is shown that the computational complexity is of $O(n)$ to compute $\hat{\b f}$ \cite{li2022convolutional}.

\subsection{Permutation Equivariance}
\label{subsec:PE}

Fix $\g=(\v,\w)$ and $\p_K$. Denote our construction of graph framelets in \cref{thm1} by $\A$ where it is provided a graph $\g$ and a corresponding hierarchical partition $\p_K$ and then builds the graph framelet $\A (\g , \p_K) = \mathcal{F}_{j_0}(\p_K)$. Let $\per: \v\rightarrow\v$ be a reordering (relabelling, bijection) of $\v=\{1,2,\ldots,n\}$, i.e., $\per$ is w.r.t. a {\em node permutation} on $[n]$ with $\pi(V)=\{\pi(1),\ldots,\pi(n)\}$. We  denote $\per(\g)=(\pi(V),\pi(\w))$  with $\pi(\w):=\{(\pi(i),\pi(j)) : (i,j)\in \w\}$. The corresponding signal $\b f$ on the graph $\g$ is reordered to be $\per(\b f)$ under the newly ordered graph $\per(\g)$. In other words, given a $\per$, there exists a permutation matrix $\b P_\per$ of size $n\times n$ such that $\per(\b f)= \b P_\per \b f$.
For each node permutation $\per$, the construction $\A$ is called (node) {\em permutation equivariant} if $\A\left(\per (\g), \p_K\right) = \per\left(\A(\g, \p_K)\right)$, where $\per(\b u_m) =\b P_\per \b u_m$ for $\b u_m\in \mathcal{F}_{j_0}(\p_K)$.

Note that $\p_K$ is a {\em tree} and that the children nodes in a parent-children subtree are ordered according to the last integer in the index vectors $\b\Lambda$.  The order of nodes in such subtrees and the order of nodes in $\mathcal{V}$ are separately defined. This means a reordering of nodes in $\mathcal{V}$ does not affect the order in subtrees in $\p_K$ and vice versa. On the other hand, the reordering of tree nodes $\b\Lambda$   may result in different graph framelets. \cref{fig:pp-short} shows a simple example. %When we  permute the children $(\b\Lambda,1)=(1,1,1,1)$ and $(\b\Lambda,2)=(1,1,1,2)$ of the tree node ${\b\Lambda=(1,1,1)}$ but use the same matrix $\b B_{\b\Lambda}$ in \cref{g1}, then we  get $ \tilde{\b \psi}_{(1,1,1,1)}\neq \b \psi_{(1,1,1,1)}$ after permutation. 
Thus it is necessary to analyze the relationship of the graph framelets under such types of permutations. We say that $\per_p$ is a {\em partition permutation} on $\p_K$ if the permutation is on the children of each tree node ${\b\Lambda}$ only. 
%We call $\b S$ a sign matrix (of size $M_\g\times M_\g$) if it is a diagonal matrix with diagonal elements in $\{-1,1\}$.
For each partition permutation $\per_p$, the construction $\A$ is called {\em partition  permutation equivariant} if $\A\left(\g, \per_p(\p_K)\right) = \per_p\left(\A(\g, \p_K)\right)$, that is, there exists a permutation $\pi^*$ on $[M_\g]$ associated with $\pi_p$ such that for each $\b u_m\in \mathcal{F}_{j_0}(\p_K)$,  $\per_p(\b u_m) = c_m\b u_{\pi^*(m)}$ for some $c_m\in\{-1,+1\}$. 
We have the following theorem regarding the permutation equivariance on both node and partition permutations. %See \cref{app:A} for more explanations.

\begin{theorem}\label{thm4}
	Let $\A (\g, \p_K)$ be the construction of the binary Haar graph framelet systems in \cref{cor}  for $j_0\in[K]$. Then, the following three statements hold:
	\begin{enumerate}
		\item[\rm(i)] For any node permutation $\per$, we have $\mathcal{A}\left( \per(\g), \p_K\right) = \per(\mathcal{A}( \g, \p_K))$.
		
		\item[\rm(ii)] For any partition permutation $\per_{p}$, we have $\mathcal{A}\left(\g, \per_p(\p_K)\right) = \per_p\left(\mathcal{A}( \g, \p_K)\right) $.
		
		\item[\rm(iii)] For any node permutation $\per$ and partition permutation $\per_p$,  we have $\mathcal{A}\left(\per(\g),  \per_p(\p_K)\right) = \per_p(\pi(\mathcal{A}( \g, \p_K)))=\pi(\pi_p(\mathcal{A}( \g, \p_K))) $.

		%		Suppose that for any permutation matrix $\b Q$ there exists a permutation matrix $\b R$ and a matrix $\b S$ which is a diagnoal matrix with diagnoal elements to be $1$ or $-1$ such that $\b A \b Q = \b S \b R \b A$. Then for any permutation matrix $\b P$ and a system $\b \b\phi$ which define a system $\b \b\phi^* = \b Q\b \b\phi \b P$, we have $\b \b\psi^* = \b S \b R \b \b\psi \b P$ where $\b \b\psi^*$ is the high frequency framelets under decomposition matrix $\b A$ of the system $\b \b\phi^*$.
	\end{enumerate}
\end{theorem}
\begin{remark} %We postpone the proof of \cref{thm4} to \cref{app:A}. 
\cref{thm4} shows that our binary framelet system $\mathcal{F}_{j_0}(\p_K)$  is permutation equivariant when reordering node or the tree indices. By applying \cref{thm4}, we show that our proposed GFNN model PEGFAN has the property of permutation equivariance. See \cref{prop:model} in the next section. %Hence the construction developed is robust for machine learning tasks.

Permutation equivariance is a subtle property that most of the GNNs in the literature possess since they generally employ operation that only involves the adjacency matrices, the graph Laplacians, summation, and concatenation. Nonetheless, there are works \cite{maroninvariant,morris2022speqnets} that theoretically investigate the permutation equivariance of general and specific GNNs, which is highly related to the graph classification and the importance of the topic of the expressiveness of GNNs \cite{sato2020survey,zhang2023expressive} as permutation is one of the most basic type of isomorphism on graphs. In this paper, we confine ourselves to the output consistency that permutation equivariance derives as this is coherent to our context of node classification. On the contrary, graph wavelets/framelets, especially Haar-type graph wavelets/framelets are more complicatedly generated mathematical tools and the discussion of such property is missing in both the mathematical literature and the recent works of GNNs that apply graph wavelets/framelets. In some of the works of Haar-type graph wavelets/framelets (\cite{12,H2,H3}), it is obvious that the permutation equivariance is violated if there are no further constraints. %In contrast, our work is the first one to theoretically discuss and prove the permutation equivariance of Haar-type graph wavelets/framelets and the proposed GNN model \cref{prop:model}.
\end{remark}

\section{GFNN Model}
\label{sec:model}

% \begin{figure*}[htpb!]
% 	\centering
% 	%\fbox{\rule{0pt}{2in} \rule{0.9\linewidth}{0pt}}
% 	\includegraphics[width=1.0\linewidth]{figures/nn.pdf}
	
% 	\caption{Neural network architecture. The input is the feature matrix $\b X$. The network operations are determined by the underlying adjacency matrix $\b A$ and the constructed binary Haar graph framelet system $\{\b F_0,\ldots, \b F_{K-1}\}$.}
% 	\label{fig:nn}
% \end{figure*}

\begin{figure*}[htpb!]
	\centering
	\includegraphics[width=0.8\linewidth]{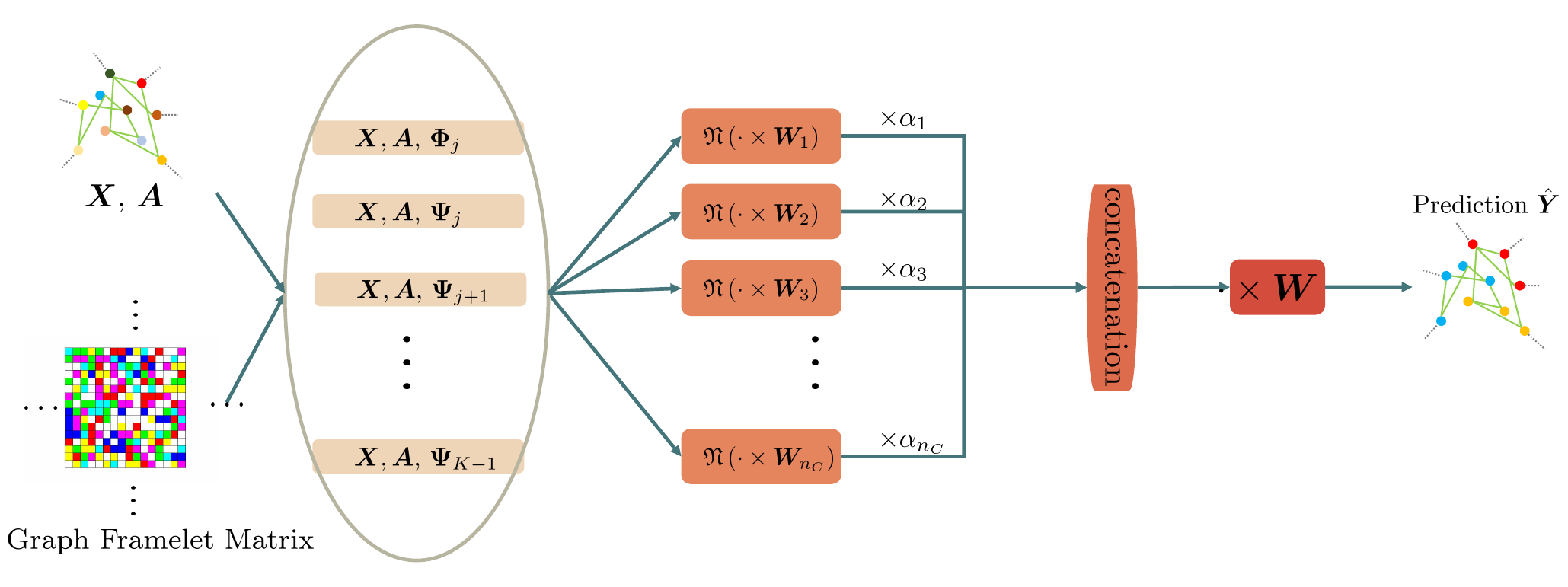}
	
	\caption{Neural network architecture. The input is the feature matrix $\b X$. The network operations are determined by the underlying adjacency matrix $\b A$ and the constructed binary Haar graph framelet system $\{\b F_0,\ldots, \b F_{K-1}\}$. The operator $\mathfrak{N}(\cdot)$ is defined as normalizing each row of any given matrix.   }
	\label{fig:nn}
\end{figure*}

We introduce the GFNN model that integrates our constructed binary Haar graph framelets, which we call \textbf{Permutation Equivariant Graph Framelet Augmented Network (PEGFAN)}, see \cref{fig:nn}.

Semi-supervised learning is characterized by involving both unlabeled and labeled data to infer a discriminative function $f$. In contrast, in supervised learning, only labeled data is utilized in obtaining $f$. In a (semi-supervised) node classification task, we assume that the first $l$ nodes are labeled. Each node $i\in \mathcal{V}$ is associated with a feature vector $\b{x}_i \in \mathbb{R}^{n_f}$ and a one-hot $\b{y}_i \in \mathbb{R}^{n_c}$ indicating the ground truth of labels, where $n_f$ and $n_c$ are the numbers of features and classes. Stacking these vectors gives a feature matrix $\b{X}\in\mathbb{R}^{n\times n_f}$ and a label matrix $\b{Y}\in \mathbb{R}^{n\times n_c}$ (the first $l$ elements are given labels and the rest part has no label and need to predict). Suppose there are $n_C$ channels, associating a series of matrices $\b{X}_1,\b{X}_2,\dots,\b{X}_{n_C}$ for each channel, and $\b{X}_i \in \mathbb{R}^{n\times d_i},1\leq i \leq n_C$. Our model is a two-layer network, which is defined as
\begin{align}
	\b{H}_1 &= \overset{n_C}{\underset{i=1}{\parallel}}\alpha_i\cdot\mathfrak{N}(\b{X}_i\b{W}_i),\\
	\hat{\b{Y}}& = \text{softmax}(\text{ReLU}(\b{H}_1)\b{W}),
\end{align}
where $\parallel$ denotes the concatenation operation, $\alpha_i$ are trainable attention weights satisfying $\alpha_i\in(0,1)$ and $\sum_i \alpha_i = 1$, $\mathfrak{N}(\cdot)$ is the row normalization operation, and $\b{W}_i\in \mathbb{R}^{d_i\times n_h}$ and $\b{W}\in \mathbb{R}^{n_Cn_h\times n_c}$ are trainable parameters. Our model comprises several input channels at the beginning and subsequently several fully connected layers. Therefore, it is easy to be extended with more layers. As usual, we minimize the cross entropy of the labeled nodes using the first $l$ columns of $\hat{\b{Y}}$ and $\b{Y}$.

Given the binary Haar graph framelet system $\mathcal{F}_{j_0}(\p_K)$,
We also use $\b \Phi_j=(\b\phi_{  {\b\Lambda}})_{\dim( {\b \Lambda} )=j} \in \RR^{N_j \times n}$ and $\b \Psi_j=(\b\psi_{(  {\b\Lambda}, m)})_{\dim( {\b\Lambda}) = j, m \in [M_{ {{\b\Lambda}}}]}  \in \RR^{M_j \times n}$ to be the matrix representations of the scaling vectors and framelet vectors at scale $j$, respectively. We denote $\b{F}_0(\b{M}):=\b{\Phi}_1^\top\b{\Phi}_1\b{M}$, $\b{F}_j(\b{M}):=\b{\Psi}_j^\top\b{\Psi}_j \b{M},1\leq j \leq K-1$.   For our GFNN model PEGFAN,  we select 3 options for $\{\b X_1,\ldots, \b X_{n_C}\}$ of feature matrices for graphs with homophily and heterophily, respectively. 

For {\em homophilous graphs}, we have 3 types:
\begin{itemize}
	\item[a)]
	$\ab n_C \ab= \ab1 +\ab K, \{\ab\b{X},\ab\b{F}_0(\b{X}),\ab\b{F}_1(\b{X}),\ab\dots,\ab\b{F}_{K-1}(\b{X})\}$.
	
	\item[b)] $\ab n_C = \ab1 + \ab r + \ab K, \{\ab\b{X},\ab\tilde{\b{A}}\b{X}\ab,\tilde{\b{A}}^2\b{X},\ab\dots, \ab\tilde{\b{A}}^r\b{X}, \ab\b{F}_0(\b{X}),\ab\b{F}_1(\b{X}),\ab\dots,\ab\b{F}_{K-1}(\b{X})\}$.
	
	\item[c)] $n_C = 1 + r + K$, $\{\ab \b{X},\ab\tilde{\b{A}}\b{X},\ab\tilde{\b{A}}^2\b{X},\ab\dots, \ab\tilde{\b{A}}^r\b{X},  \ab\b{F}_0(\tilde{\b{A}}\b{X}),\ab\b{F}_1(\tilde{\b{A}}\b{X}),\ab\dots,\b{F}_{K-1}(\tilde{\b{A}}\b{X})\}$.
\end{itemize}

For {\em heterophilous graphs}, we have 3 types:
\begin{itemize}
	\item[a)]
	$\ab n_C =\ab 1 + \ab K, \{\b{X},\ab\b{F}_0(\b{X}),\ab\b{F}_1(\b{X}),\ab\dots,\b{F}_{K-1}(\b{X})\}$.
	\item[ b)]
	$n_C \ab= 1 +\ab r +\ab K, \{\ab\b{X},\ab\b{A}\b{X}\ab,\b{A}^2\b{X},\ab\dots,\ab \b{A}^r\b{X},\ab \b{F}_0(\b{X}),\ab\b{F}_1(\b{X}),\ab\dots,\ab\b{F}_{K-1}(\b{X})\}$.
	\item[ c)]
	$ n_C =  1 +  r + K$, $\{\ab\b{X},\ab\b{A}\b{X},\ab\b{A}^2\b{X},\ab\dots, \ab\b{A}^r\b{X}, \ab\b{F}_0(\b{A}\b{X}),\ab\b{F}_1(\b{A}\b{X}),\ab\dots,\ab\b{F}_{K-1}(\b{A}\b{X})\}$.
\end{itemize}

With the permutation equivariance of our graph Haar framelets, now we can formally state the permutation equivariance of our GFNN model PEGFAN.

\begin{proposition}\label{prop:model}
Let $\g=(\v,\w)$ be a graph with  feature matrix $\b X$,  adjacency matrix $\b A$, and a $K$-hierarchical partition $\p_K$. Let  $\b P$ be a permutation matrix w.r.t. to a node permutation $\per$ on $\v$. If the permuted feature matrix $\b{P}\b{X}$, adjacency matrix $\b{P}\b{A}\b{P}^\top$,  and binary Haar graph framelet system $\per(\A(\g,\p_K))$ are used in forming the type a), b) and c) channels for PEGFAN, then the new output $\hat{\b{Y}}_{\b{P}}$ differs from the original one by a permutation matrix, i.e. $\hat{\b{Y}}_{\b{P}} = \b{P}\hat{\b{Y}}$.
\end{proposition}

%\begin{proof}
%From \cref{thm4} (i), we know that the corresponding permuted versions of $\b{\Phi}_1$ and $\b{\Psi}_i$ are $\b{P}\b{\Phi}_1$ and $\b{P}\b{\Psi}_i$. Thus $\b{F}_i(\b{P}\b{X})\ab = \ab\b{P}\b{F}_i(\b{X})$, $\ab \b{F}_i((\b{P}\b{A}\b{P}^\top)^j\b{P}\b{X})\ab = \ab \b{P}\b{F}_i(\b{A}^j\b{X})$, \ab $\b{F}_i((\b{P}\tilde{\b{A}}\b{P}^\top)^j\b{P}\b{X})\ab =\ab \b{P}\b{F}_i(\tilde{\b{A}}^j\b{X})$ for $\ab i=0,\ab\dots,\ab k-1$, $\ab j=\ab 1,\ab\dots,\ab r$. It is obvious that the remaining channels also differ by a permutation matrix $\b{P}$. Since the row normalization and the softmax function are applied row-wise and the activation function is applied element-wise, it is straightforward to see that $\hat{\b{Y}}_{\b{P}} = \b{P}\hat{\b{Y}}$.
%\end{proof}

\begin{remark}\label{rem:FSGNN}
In contrast to our PEGFAN, the model FSGNN \cite{5} adopts the 2-layer network model with the following 3 options of input channels: 1) Homophily: $n_C = 1 + r$, $\{\b{X}, \tilde{\b{A}}\b{X},\tilde{\b{A}}^2\b{X},\dots, \tilde{\b{A}}^r\b{X}\}$. 2) Heterophily: $n_C = 1 + r$, $\{\b{X}, \b{A}\b{X},\b{A}^2\b{X},\dots,\b{A}^r\b{X}\}$. 3) All: $n_C\ab  = 1+2r$, $\{\b{X},\b{A}\b{X},\tilde{\b{A}}\b{X},\b{A}^2\b{X},\tilde{\b{A}}^2\b{X},\dots,\b{A}^r\b{X},\tilde{\b{A}}^r\b{X}\}$.

As shown in \cref{fig:nn}, our network model differs from existing GNNs using graph wavelets/framelets in the sense that we fully utilize the multi-scale property of our Haar graph framelets as well as the powers of the adjacency matrix as \textbf{the multi-channel inputs}. In such a way, short- and long-range information of the graph are fully exploited for the training of the network model. On the contrary, neural network architectures of other existing GNNs using graph wavelets/framelets are similar to classical spectral graph neural networks, which are essentially different from ours in exploiting multi-scale information.
\end{remark}

% In the experiments, the claimed superiority is observed on synthetic datasets and graph node classification tasks.
%
%The current state-of-the-art model FSGNN, i.e., Feature Selection GNN (see \cite{5} and Section~\ref{sec:model} for more details),  is different from the previous ones by, in our interpretation, viewing the semi-supervised setting as a supervised setting in which multi-hop aggregation was regarded as input of different feature channels from different hops and were not applied in the following layers. As a result, its network architecture basically consists of a Mix-hop \cite{6} layer (using the products of powers of adjacency matrices and feature matrix as channels) and fully-connected layers with the addition of attention weights corresponding to different channels that apply before the concatenation in the Mixhop layer and works as a feature selection mechanism.  In addition, it adopts a more fine-tuned training that uses separated learning rates and weight decay coefficients for parameters in different components. It is worth mentioning that a recent work on large-scale heterophily node classification \cite{LINKX} is also very similar to \cite{5} in which input channels were limited to $0$-hop and $1$-hop.

\section{Experiments}\label{EX}

\subsection{Experiment on Synthetic Dataset }

%\begin{figure}[htpb]
%	\centering
%	%\fbox{\rule{0pt}{2in} \rule{0.9\linewidth}{0pt}}
%	\includegraphics[width=0.8\linewidth]{figures/line.eps}
%	\caption{Test accuracy on synthetic datasets.}
%	\label{fig:gamma}
%\end{figure}

In \cite{ISGCN}, it has been theoretically shown that for a linear classifier, using $\b{A}_{rw}:=\b{D}^{-1}\b{A}$ to aggregate features has a lower probability to misclassify under the condition that the ``neighborhood class distributions'' are distinguishable. To elaborate, it assumes that for each node $i$ of class $y_i=c$, the neighbors of $i$ are sampled from a distribution $\mathcal{D}_{y_i}$, and the distributions $\mathcal{D}_c$'s are different. For heterophilous graphs, it is possible to fit the aforementioned condition as long as for each node of some class, the connection pattern with nodes from each class is different from the patterns of nodes of a different class. In other words, using simple neighborhood aggregation such as $\b{A}_{rw}\b{X}$ in GNNs still has the chance to achieve good performance for heterophilous graphs and the experiments in \cite{ISGCN} has empirically validated this statement.

%This explains the good performance for FSGNN on datasets Chameleon and Squirrel, even though different aggregations, namely $\b{A}$ and $\tilde{\b{A}}$, are applied.

Following their observation, we are interested in how the neighborhood distribution $\mathcal{D}_{c}$ affects the performance of FSGNN and PEGFAN. We follow the way in \cite{ISGCN} and generate $4$-class heterophilous graphs with 3,000 nodes, fixed Gaussian features, and different neighborhood class distributions. The proportion of training, validation, and test set was set to 48\%, 32\%, and 20\%, respectively. We compare the performance of PEGFAN with FSGNN to demonstrate the ability of multi-scale extraction when our binary Haar graph framelet system is added. To emphasize the difference between graph framelets and $K$-hop aggregation, we \textbf{excluded} the feature matrix channel $\b{X}$ in the overall channels. A hyperparameter {$\gamma\in [0,1]$} indicates the tendency to sample edges from uniform neighborhood class distribution. Consequently, larger $\gamma$ results in more indistinguishable neighborhood class distributions. Implementation details are the same as shown in the subsection for the benchmark datasets except that the hyperparameter search range is reduced and $h$ is set to $4$, $8$, and $12$ (cf. \cref{fast1}). More details of the synthetic dataset experiment  are given in \cref{appen_d}. 

\cref{tab:normal1} collects results from the experiment following the procedure defined in \cref{appen_d}. \cref{tab:normal2} contains results of replacing features sampled Gaussian distributions with closer means, which are more similar for different classes and more difficult to classify. %Both tables still show that with addition of our Haar graph framelets, the model is able to perform better.

\begin{table}[htpb!]
	\centering
	%\scriptsize
	
	\caption{Classification Accuracy on synthetic dataset with features sampled from $6(-0.75 + 0.5 c)+\xi$, where $\xi \sim N(0,1)$, $c \in \{0,1,2,3\}$.}
	\label{tab:normal1}
	\begin{tabular}{lllllll}
		\toprule
		$\gamma$ &0&0.2&0.4&0.6&0.8&1  \\ \midrule
		Ours(Type a, $h = 12$)&68.3&68.5&63.8&59.7&63.3&65.3\\
		Ours(Type a, $h = 8$)& 67.3 &63.2& 61.2&63&61&59.2\\
		Ours(Type a, $h = 4$)&63 &56.3 &43.5 &37.7 &31.8 &27.7\\
		Ours(Type b, $h = 4$)&95.7&92.3&86&74.3&61&54  \\ 
		Ours(Type c, $h = 4$)&91.8&83.8&74.2&63.5&52.3&47.2 \\
		\hline
		FSGNN($r = 3$)   &91.7&82.3&74.2&61.5&51.8&46.8 \\
		FSGNN($r = 3$, all) &92&83.5& 74.8&63 &51.7&48\\
		FSGNN($r = 8$)&92.5&85.5&73.8&62.5&51.2&47.3\\
		FSGNN($r = 8$, all)&92.3&84.3&74.7&63&52.5&47.2\\
		\bottomrule
	\label{T1}
	\end{tabular}
\end{table}

\begin{table}[htpb!]
	\centering
	%\scriptsize
	
	\caption{Classification Accuracy on synthetic dataset with Features sampled from $(-0.75 + 0.5 c)+\xi$, where $\xi \sim N(0,1)$, $c \in \{0,1,2,3\}$.}
	\label{tab:normal2}
	\begin{tabular}{lllllll}
		\toprule
		$\gamma$ &	0&	0.2&0.4&0.6&0.8&1  \\ \midrule
		Ours(Type a, $h = 12$)&63.5&60&55.7&59.3&54.8&57.3\\
		Ours(Type a, $h = 8$)& 61.7& 57.5& 57 &55.8&55.3&55.8\\
		Ours(Type a, $h = 4$)& 58.8&50.7&42&36&28.7&23.2\\
		Ours(Type b, $h = 4$)&91.3&88.5&	77.5&62.7&53.7&42\\
		Ours(Type c, $h = 4$)&81.8&76.3&68.3&59.5&47.3&39.2\\
		\hline
		FSGNN($r = 3$)	&91.8&80.8&70.7&64.5&50.5&44.2 \\ 
		FSGNN($r = 3$, all)&92& 83.8&73.5&63.3&52.3&44.2\\
		FSGNN($r = 8$)&93.8&84&67.8&59.8&48.5&40.8\\
		FSGNN($r = 8$, all)&92.7&81.7&69.7&59.7&51.2&42.7\\
		\bottomrule
	\label{T2}
	\end{tabular}
\end{table}

\begin{figure}[htpb]
	\centering
	\subfloat[]{\includegraphics[width=6.5cm]{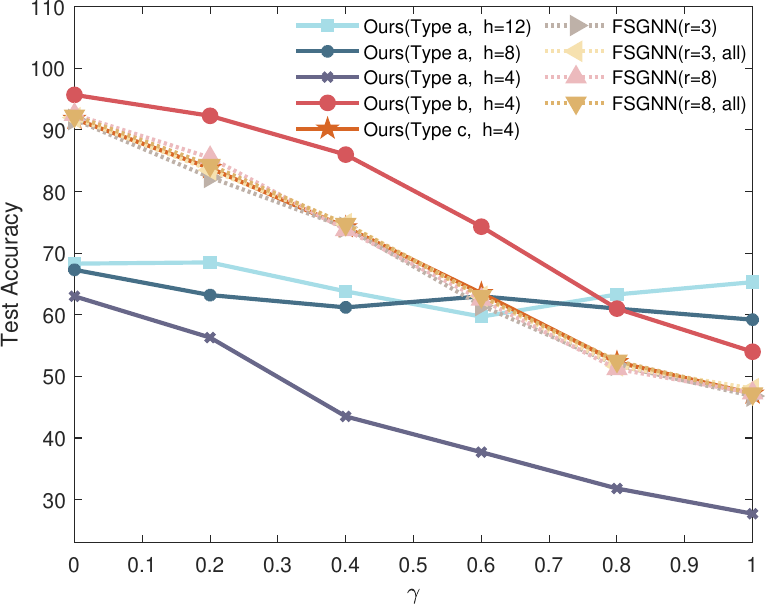}}\\
	\subfloat[]{\includegraphics[width=6.5cm]{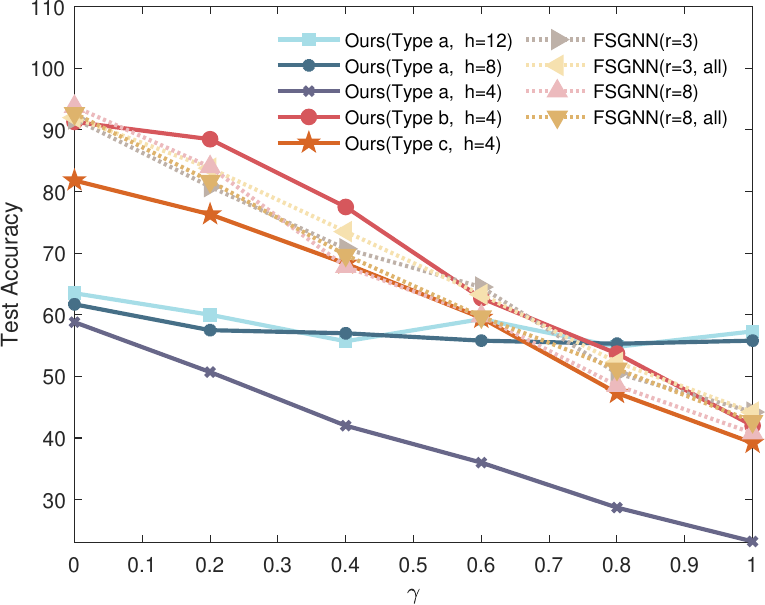}}
	    \caption{Demonstration of statistical performance comparison: (a) results in \cref{T1}; (b): results in \cref{T2}. }
	\label{LineGraph}
\end{figure}

\subsection{Comments on the Synthetic Dataset Experiment}
For the synthetic dataset, it can be seen from \cref{T1} and \cref{T1} that in most of the cases, the performance decreases with respect to $\gamma$ in all cases ({if $\gamma\rightarrow 1$ then, the graph is close to the case of being generated by uniform neighborhood class distribution}). 
%{\em Type c channels} input performs similarly to FSGNN under the case that graph framelets use the {feature matrix $\b A\b{X}$}. 
However, in \cref{T1}, {\em Type b channels} input, in which graph framelets project the original feature matrix {$\b{X}$}, show a large improvement compared with FSGNN. In some cases, the increment can reach over 10\%. This evidently shows the effectiveness of multi-scale extractions via graph framelets when combined with adjacency matrix aggregations. There are also drawbacks, which can be seen from the results of {\em type a channels input} in both tables. It shows that the results are sensitive to hyperparameter $h$ and that using graph framelets alone is not enough. Indeed, we chose a rather simple and unsupervised way to generate hierarchical clustering. This process altered the representation of the connectivity among nodes and caused a loss of information. Therefore, it is better to combine fine-scale information using $1$ to $3$-hop aggregation and coarse scale projection via graph framelets. However, the performance of {\em Type a} channels has less variation across different $\gamma$ and is better when $\gamma$ is closed to $1$. It is also obvious that the neighborhood distributions affect our model performance for {\em Type b} and {\em Type c} given the theory in \cite{ISGCN}, in which the model accuracy decreased as the neighborhood distributions approached the same and indistinguishable uniform distribution. As for {\em Type c channels} in both tables, the channels are affected by the adjacency matrix before multi-scale extraction and thus they perform similarly compared with FSGNN. In \cref{T2}, since it is more difficult to correctly classify nodes, {\em Type b channels}  gain less improvement as compared to \cref{T1}.

\subsection{Experiments on Benchmark Datasets}

\begin{table*}[htpb!]
	\scriptsize
	%\tiny
	\centering
	\caption{Dataset statistics, classification accuracy, and standard deviation. \textbf{Best} in bold, {\color{blue}{second best}} in blue.}
	
	\setlength\tabcolsep{3.5pt}
	\label{tab:InforOfDatasets}
	\begin{tabular}{l|ccc|ccc|c|cc|cl}
		\toprule
		& Cora & Citeseer & Pubmed & Texas & Wisconsin & Cornell & Actor & Chameleon & Squirrel & Avg. & Rank \\ \midrule
		Node                 &2,708 &3,327 &19,717 &183 &251 &183 &7,600  &2,277 &5,201      & -  & -  \\
		$\Vert \b{A}\Vert_0$  &10,556 &9,228 &88,651 &325 &515 &298 &30,019  &36,101 &217,073 & - & -  \\
		Feature              &1,433 &3,703 &500 & 1,703 &1,703 &1,703 &931 &2,325 &2,089                  & - & -                      \\
		Density              &$1.44\cdot 10^{-3}$&$8.34\cdot 10^{-4}$&$2.28\cdot10^{-4}$&$9.70\cdot 10^{-3}$ &$8.17\cdot 10^{-3}$ &$8.90\cdot 10^{-3}$ &$5.20\cdot 10^{-4}$ &$6.96 \cdot 10^{-3}$ &$ 8.02 \cdot 10^{-3}$                                                   & - & -\\
		Class                &6 &7 &3 &5 & 5& 5&5 &5 & 5  & -& -                                                                                    \\
		Type                 &Homophily &Homophily &Homophily &Heterophily &Heterophily &Heterophily &Heterophily &Heterophily &                                                                                     Heterophily & -& -\\ \midrule
		Mixhop         & 87.61$\pm$0.85 & 76.26$\pm$1.33 &  85.31$\pm$0.61 &77.84$\pm$7.73 &75.88$\pm$4.90 & 73.51$\pm$6.34 & 32.22$\pm$2.34 & 60.50$\pm$2.53&43.80$\pm$1.48 & 68.10 & 14\\
		GEOM-GCN       &85.27  &  \textbf{77.99} & {\color{blue}{90.05}}      & 67.57      &  64.12          & 60.81      & 31.63      &  60.90         &  38.14         & 64.05 & 15 \\
		GCNII          &88.01$\pm$1.33 &77.13$\pm$1.38&\textbf{90.30$\pm$0.37} & 77.84$\pm$5.64 &81.57$\pm$4.98& 76.49$\pm$4.37 &  -    & 62.48$\pm$2.74    &-      &  -&  -\\
		H2GCN-1        &86.92$\pm$1.37 & 77.07$\pm$1.64&89.40$\pm$0.34 &84.86$\pm$6.77 &86.67$\pm$4.69& 82.16$\pm$4.80 &35.86$\pm$1.03 &57.11$\pm$1.58 & 36.42$\pm$1.89 & 70.72 &  13 \\
		WRGAT          &{\color{blue}{88.20$\pm$2.26}}&76.81$\pm$1.89&88.52$\pm$0.92& 83.62$\pm$5.50& 86.98$\pm$3.78 &81.62$\pm$3.90 & \textbf{36.53$\pm$0.77} &65.24$\pm$0.87 & 48.85$\pm$0.78 & 72.93 & 11 \\
		GPRGNN         &\textbf{88.49$\pm$0.95}&77.08$\pm$1.63&88.99$\pm$0.40& 86.49$\pm$4.83&85.88$\pm$3.70&81.89$\pm$6.17& {\color{blue}{36.04$\pm$0.96}}&66.47$\pm$2.47 &49.03$\pm$1.28 & 73.37 & 10 \\
		\hline
		FSGNN($r=3$)     &86.92$\pm$1.66&77.18$\pm$1.27&89.71$\pm$0.45&84.51$\pm$4.71&\textbf{87.84$\pm$3.37}&84.86$\pm$4.56&35.26$\pm$1.01&78.60$\pm$0.71&73.93$\pm$2.00    & 77.65 & 5\\
		FSGNN($r=8$)     &88.15$\pm$1.15&77.23$\pm$1.41&89.67$\pm$0.45&\textbf{86.76$\pm$3.72}&{\color{blue}{87.65$\pm$3.51}}&85.95$\pm$5.10&35.22$\pm$0.96&79.01$\pm$1.23&73.78$\pm$1.58 & 
		\color{blue}{78.16} & \color{blue}{2$^{\mathrm{nd}}$} \\
		FSGNN($r=3$, all) &87.59$\pm$1.03&76.91$\pm$1.60&89.68$\pm$0.37&84.60$\pm$5.41&86.67$\pm$2.75&{86.22$\pm$6.78}&35.51$\pm$0.89&77.68$\pm$1.10&73.79$\pm$2.32          &   77.63 & 6\\
		FSGNN($r=8$, all) &87.53$\pm$1.37&76.86$\pm$1.49&89.73$\pm$0.40&82.70$\pm$5.01&85.88$\pm$5.02&85.13$\pm$7.57&35.28$\pm$0.79&77.94$\pm$1.17&74.04$\pm$1.51 &    77.23& 8\\
	\hline
		Ours($h=4$, Type a)   &79.48$\pm$2.68&71.29$\pm$2.01&88.46$\pm$0.35 &83.78$\pm$5.54&86.08$\pm$4.34           &85.95$\pm$5.51   &34.96$\pm$1.24       &65.83$\pm$2.05          &51.98$\pm$1.98         &   71.97 & 12\\
		Ours($h=4$, Type b)   &87.12$\pm$0.91&{\color{blue}{77.39$\pm$1.28}}&89.62$\pm$0.25        &{\color{blue}{86.47$\pm$5.54}}&86.67$\pm$3.59     &85.14$\pm$5.57         &35.07$\pm$1.03&79.63$\pm$1.23     &73.89$\pm$1.89         &   77.88 & 4 \\
		Ours($h=4$, Type c)   &87.36$\pm$1.09 &76.78$\pm$1.51          &89.55$\pm$0.32&85.14$\pm$4.05&87.65$\pm$4.02&\textbf{86.76$\pm$5.33}& 35.41$\pm$0.82&{\color{blue}{79.65$\pm$1.33}}&{\color{blue}{74.58$\pm$2.07}}          & 78.10 & 3$^{\mathrm{rd}}$ \\ 
%		Ours($h=6$, Type a)   &83.159$\pm$1.857 &73.505$\pm$1.664  &88.846$\pm$0.302  &85.135$\pm$3.676 &86.667$\pm$3.802 &84.054$\pm$6.098 & & &          \\ 
%		Ours($h=6$, Type b)   &87.364$\pm$1.049 &76.761 $\pm$ 1.400&  &85.135$\pm$5.303 &85.686$\pm$3.287 &84.595$\pm$5.412 & & &         \\ 
%		Ours($h=6$, Type c)   & &          &  &86.216$\pm$3.299                         &86.667$\pm$4.278 &86.216$\pm$4.751 & &80.20$\pm$0.71 & 75.25$\pm$2.40        \\ 
		Ours($h=8$, Type a)   &83.16$\pm$1.86 &73.51$\pm$1.67  &88.85$\pm$0.30  &84.32$\pm$3.78 &86.67$\pm$3.80 &84.05$\pm$6.10 &35.15$\pm$0.77&77.48$\pm$1.71&71.10$\pm$1.75        & 76.03& 9 \\ 
		Ours($h=8$, Type b)   &87.22$\pm$1.21 &76.76$\pm$1.40  &89.73$\pm$0.40  &84.87$\pm$5.70 &85.69$\pm$3.29 &84.60$\pm$5.41 &35.34$\pm$0.81&79.21$\pm$1.09 &73.09$\pm$1.66        &  77.39 & 7 \\ 
		Ours($h=8$, Type c)   &87.16$\pm$1.31 &76.92$\pm$1.57  &89.56$\pm$0.30  &86.22$\pm$3.30 &86.67$\pm$4.28 &{\color{blue}{86.22$\pm$4.75}} &35.48$\pm$0.94 & \textbf{80.31$\pm$1.10} & \textbf{75.06$\pm$1.72}        & \textbf{78.18}& {\bf 1}$^{\tiny\mathrm{st}}$ \\ 
		\bottomrule
	\end{tabular}
	
\end{table*}

%%%%%%%%%%%%%%%%%%%%%%%%%%%%%%%%%%%%%%%%%%%%%%%%%%

We conducted experiments on 9 datasets including 3 homophilous citation networks and 6 heterophilous datasets and followed the public data splits provided in \cite{GEOMGCN}.  
%{\xz When inspecting the adjacency matrices of the datasets, we found that some of them have self-loops, i.e. non-zero diagonal terms, and the actor dataset is directed.} 
We define the density of a graph as $\Vert \b{A}\Vert_0/n^2$, which is the proportion between the number of non-zero terms in $\b{A}$ and the numbers of terms of $\b{A}$.} The statistics is summarized in the top rows of Table \ref{tab:InforOfDatasets}.
%\onecolumn
%\begin{table}[]
%\centering
%\begin{tabular}{@{}lllllcc@{}}
%\toprule
%          & Node  & $\Vert \b{A}\Vert_0$ & Feature & Density                & Class & Type        \\ \midrule
%Cora      & 2,708  & 10,556                                                 & 1,433    & $1.44\times 10^{-3}$   & 6     & Homophily   \\
%Citeseer  & 3,327  & 9,228                                                   & 3,703    & $8.34\times 10^{-4}$   & 7     & Homophily   \\
%Pubmed    & 19,717 & 88,651                                                   & 500     & $2.28\times 10^{-4}$   & 3     & Homophily   \\
%Texas     & 183   & 325                                                             & 1,703    & $9.70\times 10^{-3}$   & 5     & Heterophily \\
%Wisconsin & 251   & 515                                                             & 1,703    & $8.17\times 10^{-3}$   & 5     & Heterophily \\
%Cornell   & 183   & 298                                                             & 1,703    & $8.90\times 10^{-3}$   & 5     & Heterophily \\
%Actor     & 7,600  & 30,019                                                           & 931     & $5.20\times 10^{-4}$   & 5     & Heterophily \\
%Chameleon & 2,277  & 36,101                                                           & 2,325    & $6.96 \times 10^{-3}$  & 5     & Heterophily \\
%Squirrel  & 5,201  & 217,073                                                          & 2,089    & $ 8.02 \times 10^{-3}$ & 5     & Heterophily \\ \bottomrule
%\end{tabular}
%\caption{Dataset statistics}
%\label{tab:InforOfDatasets}
%\end{table}
%\twocolumn
To generate a series of partitions for each dataset, we applied $\textit{sknetwork.hierarchy.Ward}$ and $\textit{sknetwork.hierarchy.cut\_balanced}$ from python package $\textit{scikit-network}$\footnote{https://scikit-network.readthedocs.io/en/v0.26.0/} to form intermediate clusters and control the hyperparameter $h$ in \cref{fast1}. $h$ is set to $4$ and $8$ and the values are indicated in \cref{tab:InforOfDatasets}. Once new partition $\v''$ of clusters is formed from a   graph $G' = (\mathcal{V}',\b{A}')$, we  define as follows the new adjacency matrix $\b{A}''$  to form the graph $G''=(\mathcal{V}'',\b{A}'')$ for next level clustering:
		\[
			\b{A}''_{ij} =\sum_{p=1}^{n'}\sum_{q=p+1}^{n'} \b{A}'_{pq}\delta(ID(p),i)\delta(ID(q),j),
		\]
		where $\# \mathcal{V}' = n'$, $\# \mathcal{V}'' = m'$,  $ID(p),ID(q)$ are the indices of clusters that nodes $p$ and $q$ belong to and $\delta(a,b)$ takes $1$ when $a=b$. For heterophilous graphs, we iterate for a few steps until the final graph has less than $h=4$ or $h=8$ nodes.  For {\it Pubmed}, when $h=4$ we constrained the number of steps of generating hierarchical clustering to be 6 so as to reduce input channels. 
%		The constants $4,6,8$ are related to $h$ and $K$ in \cref{fast1}, which shows that when we use the binary Haar graph Haar framelets, the total number of framelets only linearly increases with respect to the total number of nodes in the graph. In practice, this gives acceptable numbers of generated framelets for the datasets.

%Details are shown in \cref{app:EE}.  

%constrained the clusters to have no more than $4$ nodes in each clustering step.% Once new partition $\v''$ of clusters is formed from a   graph $G' = (\mathcal{V}',\b{A}')$, we  define as follows the new adjacency matrix $\b{A}''$  to form the graph $G''=(\mathcal{V}'',\b{A}'')$ for next level clustering:
%\[
%	\b{A}''_{ij} =\sum_{p=1}^n\sum_{q=p+1}^n \b{A}'_{pq}\delta(ID(p),i)\delta(ID(q),j),
%\]
%where $|\mathcal{V}'| = n$, $|\mathcal{V}''| = m$,  $ID(p),ID(q)$ are the indices of clusters that nodes $p$ and $q$ belong and $\delta(a,b)$ takes $1$ when $a=b$. For heterophily graphs, we iterate for a few steps until the final graph had less than 4 nodes. For {\it Pubmed}, we constrained the number of steps of generating hierarchical clustering to be exactly 6. The constant $4$ and $6$ are related to $h$ and $K$ in  \cref{fast1}, which shows that when we use the binary Haar graph Haar framelets, the total number of framelets only linearly increases with respect to the total number of nodes in the graph. In practice, this gives acceptable numbers of generated framelets for the datasets.
%
As for the implementation of the neural network\footnote{https://github.com/zrgcityu/PEGFAN}, we adopted the publicly released code of FSGNN\footnote{https://github.com/sunilkmaurya/FSGNN/}  for integrating the graph framelet projections as detailed in our PEGFAN model. We use the same optimizer, hidden layer size, etc., as those in FSGNN, and hence the details are omitted. We noticed that the outcome of FSGNN was a bit different from those reported in \cite{5} when we tried to reproduce the results. Therefore we did a separate grid search for FSGNN and the results had slight changes. For our model, we set $r$ of input channels to $3$. Results of other models (Mixhop\cite{6}, GEOM-GCN\cite{GEOMGCN}, GCNII\cite{GCNII}, H2GCN-1\cite{8}, WRGAT\cite{WRGAT}, GPRGNN\cite{GPRGNN}) are cited from \cite{5} and the results of some of the top rows are omitted, which are not among the models with relatively superior performance. All results are collected in Table \ref{tab:InforOfDatasets}.  

As a brief comparison, \cref{tab:traintime} summarizes the average, maximum, and minimum training time of our model and FSGNN on {\em Chameleon} and {\em Squirrel} over 108 sets of hyperparameters shown in \cref{HP} of \cref{appen_d}. Each training consists of 10 individual training, each of which is on a single data split. All experiments in this paper were conducted using an RTX 3090 graphics card.

	%We constrained the clusters to have no more than $4$ nodes in each clustering step.
%\begin{table}[hb!]
%	\centering
%	%\scriptsize
%		\caption{Training time.}
%	\label{tab:traintime}
%	\begin{tabular}{llll}
%		\toprule
%		Chamemleon &average   &maximum    &minimum  \\ \midrule
%		FSGNN(r=3) &20.71s &64.68s &11.21s  \\
%		FSGNN(r=8) &37.41s &90.7s &17.75s  \\
%		Ours(Type c) &47.66s &125.25s &22.52s  \\ \toprule
%		Squirrel &average       &maximum   & minimum  \\ \midrule
%		FSGNN(r=3) &34.57s   &83.24s  &18.58s  \\
%		FSGNN(r=8) &53.40s   &141.24s &26.52s  \\
%		Ours(Type c) &59.82s &198.77s &32.87s \\
%		\bottomrule
%	\end{tabular}
%\end{table}

\begin{table}[hb!]
	\centering
	\scriptsize
    %\scriptsize
		\caption{Training time over 108 configurations of hyperparameters. Number of Channels: FSGNN($r=3$): 4, FSGNN($r=8$): 9, Ours(Type c, $h=4$): 13 (Chameleon), 14 (Squrriel)}
		\setlength\tabcolsep{2.5pt}
	\label{tab:traintime}
	%\begin{tabular}{llll|llll}
	%	\toprule
	%	Chamemleon &avg.   &max.    &min.  &   Squirrel &avg.       &max.  & min. \\ \midrule
	%	FSGNN($r=3$) &20.71s &64.68s &11.21s & FSGNN($r=3$) &34.57s   &83.24s  &18.58s  \\
	%	FSGNN($r=8$) &37.41s &90.70s &17.75s  & FSGNN($r=8$) &53.40s   &141.24s &26.52s  \\
	%	Ours(Type c, $h=4$) &47.66s &125.25s &22.52s  & Ours(Type c, $h=4$) &59.82s &198.77s &32.87s \\ 
	%\bottomrule
	%\end{tabular}
    \begin{tabular}{llll}
    \toprule
    Chameleon &avg.   &max.    &min.\\ \midrule
    FSGNN($r=3$) &20.71s &64.68s &11.21s \\
    FSGNN($r=8$) &37.41s &90.70s &17.75s \\
    Ours(Type c, $h=4$) &47.66s &125.25s &22.52s\\ \midrule \midrule
    Squirrel &avg.       &max.  & min. \\ \midrule
    FSGNN($r=3$) &34.57s   &83.24s  &18.58s  \\
    FSGNN($r=8$) &53.40s   &141.24s &26.52s  \\
    Ours(Type c, $h=4$) &59.82s &198.77s &32.87s \\ 
    \bottomrule
    \end{tabular}
\end{table}

\begin{figure}[htpb]
	\centering
	\subfloat[]{\includegraphics[width=6.5cm]{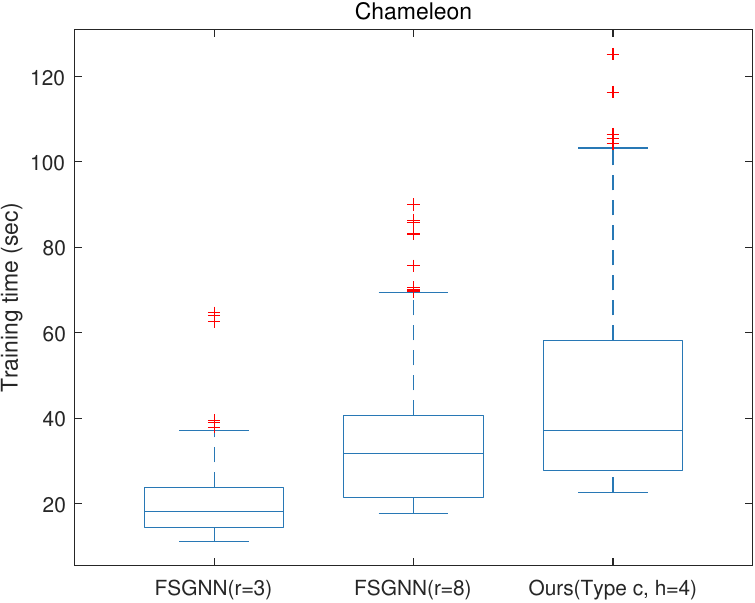}}\\
	\subfloat[]{\includegraphics[width=6.5cm]{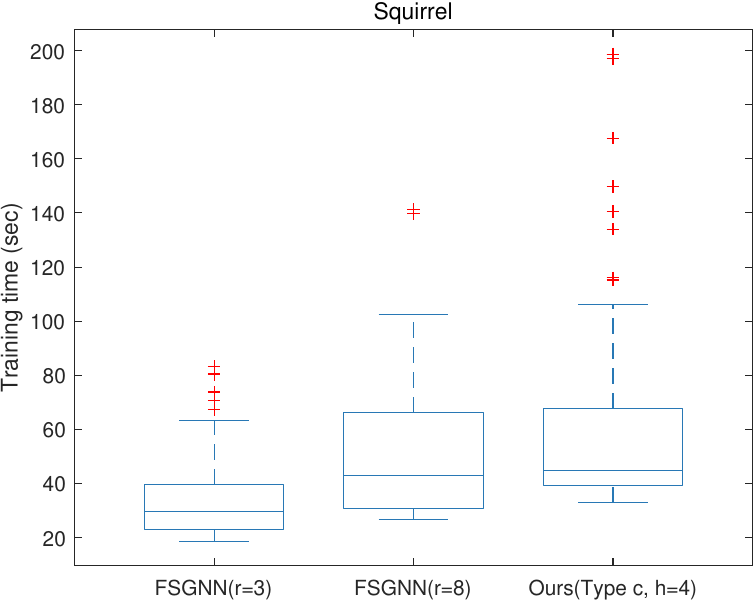}}
	    \caption{Demonstration of training time comparison: (a) Chamemleon dataset; (b) Squirrel dataset}
	\label{train_time}
\end{figure}
%\onecolumn

%\begin{figure}[t!]
%	\centering
%	%\fbox{\rule{0pt}{2in} \rule{0.9\linewidth}{0pt}}
%	\includegraphics[width=0.4\linewidth]{figures/line.eps}
%	\caption{Test accuracy on synthetic datasets.}
%	\label{fig:gamma}
%\end{figure}

\subsection{Comments on Benchmark Dataset Experiments}
We provide in \cref{tab:InforOfDatasets}  not only the performance of many state-of-the-art models but also their performance on both the homophilous and heterophilous graph datasets (9 datasets in total). 

As pointed out in the Introduction, traditional models are usually with the underlying assumption of homophily. They perform well for homophilous graph datasets. One can clearly see from  \cref{tab:InforOfDatasets} that the best performances for the three typical homophilous datasets ({\em Cora}, {\em Citeseer}, and {\em Pubmed}) are given by GEOM-GCN, GCNII, and GPRGNN. For the homophilous datasets, their nature of being homophilous does not necessitate the need for further multi-scale information, and thus our method has a similar performance. The same drawback is shown as in the results of synthetic data, where the results of {\em Type a channels} are not superior and sensitive to the hyperparameter $h$. It empirically shows that to use framelets alone, it is required to form sufficiently large clusters at the beginning of forming hierarchical partitions. 

While models such as
GEOM-GCN, GCNII, and GPRGNN perform well in those homophilous datasets, they do not give the best performance for the other six heterophilous datasets.  The models that give the best performance for heterophilous datasets are FSGNN and our PEGFAN.

Now between FSGNN and our PEGFAN, from the above discussion, we only need to focus on the 6 heterophilous datasets: {\em Texas}, {\em Wisconsin}, {\em Cornell}, {\em Actor}, {\em Chamelon}, and {\em Squirrel}. We would like to emphasize that we follow the most common way that uses the public data splits in \cite{GEOMGCN}. The proportions of train-validation-test splits are all 48\%, 32\%, 20\%. These 6 datasets can be considered as three groups discussed as follows. 

The first group is the datasets of {\em Texas}, {\em Wisconsin}, and {\em Cornell}. They are similar datasets with a small number of nodes, edges, and features. Since the test sets are only 20\% of the graph, they contain at most 51 nodes. A correctly predicted node accounts for at least 1.9\% of accuracy. Hence we can say that experiments on such datasets are relatively and statistically insignificant. Most of the models have very similar performance with at most 7 nodes wrongly predicted. Nonetheless, we chose to follow the common conduct and report the results for completeness. PEGFAN is best for {\em Cornell} while ranks second for {\em Texas}. FSGNN is best for {\em Wisconsin} and {\em Texas} while ranks third for {\em Cornell}. The best or second-best performances of FSGNN and PEGFAN are without much difference. Since the number of nodes is too small. It is not reasonable to say one is better than the other.

The second group is simply the dataset {\em Actor}. It is a large dataset with 7,600 nodes. However, for this Actor dataset, all models, including Mixhop, GEOM-GCN, GCNII, etc., do not give reasonable performance. They only give very low accuracy about 35\%. The best performance is given by the model WRGAT. For this dataset, it is not reasonable to compare performance among different methods.

The last group is the datasets of {\em Chameleon} and {\em Squirrel}. They are both big datasets in terms of nodes and edges. We can see that PEGFAN performs the best. Being heterophilous makes it necessary to gather multi-scale information, and denser graphs facilitate forming better series of partitions and thus better graph framelets. This is also consistent with the empirical results of the synthetic dataset since as shown in \cite{ISGCN}, the neighborhood distributions of {\em Chameleon} and {\em Squirrel} are distinguishable enough for different classes, while other heterophilous datasets either are small datasets that suffer from bias or do not fit such a condition.

Moreover, to compare the {\em overall performance} of each method on the 9 benchmark datasets, we take the average of the performance of each method over the 9 datasets. The average score of each method is given in the second last column (Avg.) of \cref{tab:InforOfDatasets}. Our method with respect to $h=8$ and Type c (the last row) ranks first among the comparing methods. In general, our Type c methods outperform other methods with high average overall performance. See the last column of \cref{tab:InforOfDatasets} for the ranking of each method. 

\section{Conclusion}

This paper proposes a novel and general method to construct Haar-type framelets on graphs that are permutation equivariant. It aims to serve as an alternative and supplement for multi-hop aggregations using powers of adjacency matrices. The results show that combining graph framelets and multi-hop aggregation increases the performance of node classification on heterophilous graphs in both synthetic and real-world data. Moreover, compared with using multi-hop aggregation alone, in the synthetic case our model shows significant increases against the deterioration of neighborhood distribution and results show consistency between the synthetic and benchmark datasets in terms of the patterns of neighborhood distribution. The overall results validate the capability of our graph framelets to extract multi-scale information under certain conditions and its superior performance.  Finally,  we would also like to mention that choosing a more sophisticated way to generate the hierarchical partitions has the potential to produce better graph framelets, which will be a future experimental direction to explore.

%\newpage
%\appendix
%\onecolumn

%\setcounter{page}{1}
%\begin{center}
%{\Large\bf Supplemental File}
%\end{center}

\begin{appendix}	
	\subsection{Proofs of theoretical results}
	\label{app:A}	

\begin{proof}[Proof of \thmref{thm1}]
We denote $\b \Phi_j =  \{ \b\phi_{  {\b\Lambda}} \}_{\dim (  {\b\Lambda})=j}$ and $\b \Psi_j = \{\b\psi_{(  {\b\Lambda}, m)} \}_{\dim (  {\b\Lambda})=j, m\in [M_{ {{\b\Lambda}}}]}$. Let $V_j :=\Span \b\Phi_j$ and $W_{j}:=\Span\b \Psi_j$.
%be
%		\begin{align*}
%		V_j&:=\Span \{ \b\phi_{  {\b\Lambda}}  \}_{\dim (  {\b\Lambda})=j},\qquad
%		W_j:=\Span \{\b\psi_{(  {\b\Lambda}, m)} \}_{\dim( {\b\Lambda}) = j, m \in [M_{ {{\b\Lambda}}}]},
%		\end{align*}
%		respectively.
Note that supports of $\b\phi_{  {\b\Lambda}}$ and $\b\phi_{  {\b\Lambda}'}$ are disjoint if $\b\Lambda\neq {\b\Lambda'}$, so are  $\b\psi_{(  {\b\Lambda}, m)}$ and $\b\psi_{(  {\b\Lambda}', m')}$. Hence, by definition and $\|\b p_{\b \Lambda} \| =1$, we can see that $\b\Phi_j$ forms an orthonormal basis of $V_j$ for each $j$. Thus by Lemma 1 in \cite{li2022convolutional}, the conditions
		$\b B_{ {{\b\Lambda}}} \b B_{ {{\b\Lambda}}}^\top \b B_{ {{\b\Lambda}}} = \b B_{ {{\b\Lambda}}}$, $ \b B_{ {{\b\Lambda}}} \b p_{ {{\b\Lambda}}} = \b 0 $, and $\rank(\b B_{ {{\b\Lambda}}}) = L_{ {{\b\Lambda}}}-1$ are equivalent to that $V_{j+1} = V_{j} \oplus W_j$ and $\{ \b\phi_{  {\b\Lambda}} \}_{\dim (  {\b\Lambda})=j} \cup \{\b\psi_{(  {\b\Lambda}, m)} \}_{\dim (  {\b\Lambda})=j, m\in [M_{ {{\b\Lambda}}}]}$ is a tight frame of $V_{j+1}$.  Iteratively, for $j_0<j$, we deduce that  $V_{j_0}\oplus W_{j_0}\oplus\cdots\oplus W_{j-1} = V_j$ and $\Phi_{j_0} \cup \Psi_{j_0} \cup \cdots \cup \Psi_{j-1}$ is a tight frame for  $V_j$  if and only if matrices $\b B_{ {{\b\Lambda}}}$ and vectors $\b p_{ {{\b\Lambda}}}$ satisfy $\b B_{ {{\b\Lambda}}} \b B_{ {{\b\Lambda}}}^\top \b B_{ {{\b\Lambda}}} = \b B_{ {{\b\Lambda}}}$, $ \b B_{ {{\b\Lambda}}} \b p_{ {{\b\Lambda}}} = \b 0 $, and $\rank(\b B_{ {{\b\Lambda}}}) = L_{ {{\b\Lambda}}}-1$ for all $\b\Lambda$ with $\dim( { {\b\Lambda}}) = j_0,\ldots, j$. Now  the conclusion of the theorem follows by letting $j=K$ and noting that $\mathcal{F}_{j_0}(\p_K)=\Phi_{j_0} \cup \Psi_{j_0} \cup \cdots \cup \Psi_{K-1}$ as well as $V_K = L_2(\g)$. 
\end{proof}

 \begin{proof}[Proof of \propref{B}]
If $\b B^\top \b B = c (\b I-\b p \b p^\top)$, then $\b B \b B^\top \b B = c \b B$ by direct computation and in view of $\b B \b p = \b 0$. Conversely, if $\b B \b B^\top \b B = c \b B$ for some constant $c$. Then,  by $\b B (\b B^\top \b B) =  c \b B = c \b B  ( \b I -\b p \b p^\top)$ and  $\b p^\top( \b I - \b p \b p^\top) =\b 0$,  we have $\left[\begin{matrix}\b p^\top \\ \b B\end{matrix}\right] (\b B^\top \b B- c(\b I- \b p\b p^\top)) = \b 0$. Consequently, by the full rank property of the matrix $[\b p, \b B^\top]$, we conclude that $\b B^\top \b B = c(\b I -\b p\b p^\top)$. The particular part follows by direction evaluation. 
%In particular, if $c\neq 0$,  for $\b P = [\b p, \f{1}{\sqrt{c}}\b B^\top]$, we have
%$\b P \b P^\top
%		= \b p \b p^\top + \f{1}{c} \b B^\top \b B 
%		=   \b p \b p^\top + \b I -  \b p \b p^\top = \b I$. 
we are done.
\end{proof}

\begin{proof}[Proof of \corref{cor}]
We only need to show that $\b B_{\b\Lambda}$ and $\b p_{\b\Lambda}$ satisfy
$\b B_{ {{\b\Lambda}}} \b B_{ {{\b\Lambda}}}^\top \b B_{ {{\b\Lambda}}} = \b B_{ {{\b\Lambda}}}$, $ \b B_{ {{\b\Lambda}}} \b p_{ {{\b\Lambda}}} = \b 0 $ and $\rank(\b B_{ {{\b\Lambda}}}) = L_{ {{\b\Lambda}}}-1$. Obviously, $ \b B_{ {{\b\Lambda}}} \b p_{ {{\b\Lambda}}} = \b 0 $. Define
$\b A_{\b\Lambda}:=[\b p_{\b\Lambda}, \b B_{\b\Lambda}^\top]^\top$.  By direct evaluations, one can show that the columns of $\b A_{\b\Lambda}$ satisfy  $\| [\b A_{\b\Lambda}]_{:\ell_1}\| = 1$ and their inner product $\langle [\b A_{\b\Lambda}]_{:\ell_1},[\b A_{\b\Lambda}]_{:\ell_2}\rangle=0$ for all $\ell_1\neq \ell_2$. That is,  $\b A_{\b\Lambda} ^\top \b A_{\b\Lambda} = \b I$, where $\b I$ is the identity matrix of size $L_{\b\Lambda}$. Consequently, we deduce that $\b B_{ {{\b\Lambda}}}^\top \b B_{ {{\b\Lambda}}} = \b A_{\b\Lambda} ^\top \b A_{\b\Lambda} - \b p_{\b\Lambda}\b p_{\b\Lambda}^\top = \b I -  \b p_{\b\Lambda}\b p_{\b\Lambda}^\top$, which then implies $\b B_{\b\Lambda} \b B_{\b\Lambda} ^\top \b B_{\b\Lambda} = \b B_{\b\Lambda} (\b I-\b p_{\b\Lambda}\b p_{\b\Lambda}^\top) = \b B_{\b\Lambda} $ in view of  $\b B_{\b\Lambda} \b p_{\b\Lambda}=\b 0$. Now  $\rank(\b B_{ {{\b\Lambda}}}) = L_{ {{\b\Lambda}}}-1$ directly follows from that $\b A_{\b\Lambda}$ is of full column rank and $\b B_{\b\Lambda} \b p_{\b\Lambda}=\b 0$. We are done.
\end{proof}

\begin{proof}[Proof of \cref{sparsity}]
%Denote $\b F$ the framelet matrix in which each column is the element of $\mathcal{F}_{j_0}(\p_K)$. 
We first consider the sparsity of $\langle \b I_{:1}, \b\psi_{( {\b\Lambda}, m)} \rangle$, $m=1,\dots,M_{ {{\b\Lambda}}}$. Notice that only when the node $1 \in s_{ {{\b\Lambda}}}$, can the term $\langle \b I_{:1},  \b\psi_{( {\b\Lambda}, m)} \rangle$  be nonzero. Thus, without loss of generality, we assume that $1\in s_{\b\Lambda}$. 
%By the matrix $\b B_{ {{\b\Lambda}}}$ we define, we see that $\b\psi_{(  {\b\Lambda}, m)} = \frac{1}{\sqrt{L_{\b\Lambda}}}(\b\phi_{(  {\b\Lambda},\ell_1)}-\b\phi_{(  {\b\Lambda}, \ell_2)})$ for the unique pair  $ m=m(\ell_1,\ell_2)=m(\ell_1,\ell_2,L_{\b\Lambda})$ satisfying $1\le \ell_1<\ell_2\le L_{\b\Lambda}$ as given in \cref{cor}. Again, only when the support of $\b\phi_{(  {\b\Lambda}, \ell_1)}$ or $\b\phi_{(  {\b\Lambda}, \ell_2)}$ contains $1$, can the inner product $\langle \b I_{:1},  \b\psi_{( {\b\Lambda}, m)} \rangle$  be nonzero. Now we count the total number of such pairs $(\ell_1,\ell_2)$ with $(\ell_1,\ell_2)$ mapped to $ m$ so that $\b\psi_{(  {\b\Lambda}, m)} = \frac{1}{\sqrt{L_{\b\Lambda}}}(\b\phi_{(  {\b\Lambda},\ell_1)}-\b\phi_{(  {\b\Lambda}, \ell_2)})$ with either $s_{(\b\Lambda, \ell_1)}$ or $s_{(\b\Lambda, \ell_2)}$  containing the node 1. Since $1\in s_{\b\Lambda}$ and $\{s_{(\b\Lambda,\ell_1)} : \ell_1\in[L_{\b\Lambda}]\}$ form a partition of $s_{\b\Lambda}$, we see that there exists exactly one $\ell_0\in[L_{\b\Lambda}]$ such that  $1\in s_{\b\Lambda, \ell_0}$. Then, we see that only those $m$ with $(\ell_0,\ell_2)$ mapped to $ m$ for $\ell_2=\ell_0+1,\ldots, L_{\b\Lambda}$ or $m$ with $(\ell_2,\ell_0)$ mapped to $ m$ for $\ell_2=1,\ldots,\ell_0-1$ can make the term $\langle \b I_{:1},  \b\psi_{( {\b\Lambda}, m)} \rangle$ being nonzero. The total number of such $m$ is  $L_{\b\Lambda}-1\le h-1$.  
%
Thus, by our construction in \cref{cor}, at most $h-1$ framelets $\b\psi_{(  {\b\Lambda}, m)}$ that make $\langle \b I_{:1}, \b\psi_{( {{\b\Lambda}},m)} \rangle \neq 0$.  For each $j$, only one cluster $s_{\b\Lambda}$ of $\v_j=\{s_{\b\Lambda} : \dim(\b\Lambda) = j\}$ contains node $1$. Thus $\b F^\top \b I_{:1}$ has at most $(h-1)(K-1)$ nonzero entries. Similar results hold for $\b I_{:i}$. Hence, for $\b f = [f_1,\ldots,f_n]^\top$, it is easy to show that $\|\hat {\b f}\|_0=\|\b F^\top \b f\|_0 =  \|\sum_{i\in[n], f_i\neq0} \b F^\top\b I_{:i}\|_0 \leq \sum_{i\in [n], f_i\neq 0} \|\b F^\top \b I_{:i}\|_0 \leq (h-1)(K-1)\|\b f\|_0$.
	\end{proof}

	For generating framelets, we use \cref{alg:matrix} (\cref{1,2}). Its efficiency is discussed in \cref{sparsity}.

	\begin{algorithm}[htpb!]
		\caption{Generating framelets}
		\label{alg:matrix}
		\begin{algorithmic}
			\State {\bfseries Input:} Node set $\v$, Partition $\p_K$, Vectors $ \{\b p_{\b\Lambda}\}$, Matrices $\{\b B_{\b\Lambda}\}$, $j_0$
			%\REPEAT
			\State initialize $\mathcal{F}_{j_0}(\p_K) = \emptyset$, $\phi_{  {\b\Lambda}} = \b I_{:i}$ for any $s_{\b\Lambda}=\{i\}$.
			\For{$j=2$ {\bfseries to} $j_0-1$}
			\For{$\b\Lambda \in  \{\b\Lambda:\dim(\b\Lambda)=j\}$}
			\State $\b\phi_{  {\b\Lambda}} := \sum_{\ell\in[L_{  {\b\Lambda}}]} p_{(  {\b\Lambda}, \ell)}  \b\phi_{(  {\b\Lambda}, \ell)}$
			\For{$m=1$ {\bfseries to} $M_{\b\Lambda}$}
			\State $\b\psi_{(  {\b\Lambda}, m)} := \sum_{\ell\in[L_{  {\b\Lambda}}]} \left(\b B_{ {\b\Lambda} }\right)_{m,\ell}  \b\phi_{(  {\b\Lambda}, \ell)}$
			\EndFor
			\State update $\mathcal{F}_{j_0}(\p_K) \leftarrow \mathcal{F}_{j_0}(\p_K) \cup \{\b\phi_{  {\b\Lambda}},\b\psi_{(  {\b\Lambda}, m),m=1,\dots,M_{\b\Lambda}} \}$
			\EndFor
			\EndFor
			\State {\bfseries Output:} $\mathcal{F}_{j_0}(\p_K)$
		\end{algorithmic}
	\end{algorithm}

	\begin{proof}[Proof of \cref{fast1}]
		Note that  we have $n \leq C h^{K-1}$  and $\#\v_j= \#\{\b\Lambda:\dim(\b\Lambda)=j \}  \leq C h^{j-1}$ for some fixed constant $C>0$. 	Moreover,		$M_{ {{\b\Lambda}}} = \f{L_{ {{\b\Lambda}}}(L_{ {{\b\Lambda}}}-1)}{2} \leq \f{h(h-1)}{2}$. Therefore, there is no more than $		C(h^{j_0-1}+\sum_{j=j_0}^{K-1} \f{1}{2}h(h-1)  h^{j-1}) =  O(nh)$
		 elements in the binary graph Haar framelet system $\mathcal{F}_{j_0}(\p_K)$ for any $j_0\in[K]$. By \cref{1,2}, we have
		\begin{align}\label{alg2}
		\b\phi_{  {\b\Lambda}}^\top
		:=
		\b p_{\b \Lambda}^\top
		\begin{bmatrix}
		\b\phi_{(\b\Lambda, 1)}^\top \\ \vdots \\ \b\phi_{{(\b\Lambda, L_{\b\Lambda})}}^\top
		\end{bmatrix}
		,\quad
		\begin{bmatrix}
		\b\psi_{(\b\Lambda,1)}^\top \\ \vdots \\ \b\psi_{(\b\Lambda,M_{\b\Lambda})}^\top
		\end{bmatrix}
		:=
		\b B_{\b\Lambda}
		\begin{bmatrix}
		\b\phi_{(\b\Lambda, 1)}^\top \\ \vdots \\ \b\phi_{{(\b\Lambda, L_{\b\Lambda})}}^\top
		\end{bmatrix}.
		\end{align}
{By our construction, there is at most $h^{K-j}$ nonzero entries  for each $\b\phi_{  {\b\Lambda}}$ and at most $2\cdot h^{K-j-1}$ nonzero entries for each $\b\psi_{  ({\b\Lambda},m)}$ for $\dim(\b\Lambda) = j$.}	 Hence, the number of nonzero entries of $\b F$ is at most
$C(h^{K-j_0}\cdot h^{j_0-1}+\sum_{j=j_0}^{K-1} 2 h^{K-j-1} \cdot \f{h(h-1)}{2} \cdot h^{j-1} )\le   C(K-1)h^K = O(nh\log_h n)$.
Fix a $\b \Lambda$ which has size $\dim(\b \Lambda)=j$. { Then \cref{alg2} implies at most $h \cdot h^{K-j-1}$ multiplication operations and $(h-1) \cdot h^{K-j-1}$ addition operations needed for $\b\phi_{\b\Lambda}$. }For computing $\Psi_{(\b\Lambda, m)}$, $m=1,\dots,M_{\b\Lambda}$, we need at most $ 2 \cdot  h^{K-j-1} \cdot \f{h(h-1)}{2} $ multiplication operations and $ h^{K-j-1} \cdot \f{h(h-1)}{2} $ addition operations, respectively. Notice that $\#\v_j \leq C h^{j-1}$.  To compute the nonzero entries of $\b\phi_{  {\b\Lambda}}$ and $\b\psi_{(\b\Lambda,m)}$ for all $\dim(\b\Lambda) = j$ and $m=1,\ldots, M_{\b\Lambda}$, from the above computation, one can see that it  needs at most
$2C({ h^{K-j-1} \cdot h \cdot h^{j-1}} + 2h^{K-j-1} \cdot \f{h(h-1)}{2} \cdot h^{j-1}) = 2C\cdot h^K$ evaluations of multiplications and additions.  Hence, in total,  to compute the nonzero entries of $\b\phi_{  {\b\Lambda}}$ and $\b\psi_{(\b\Lambda,m)}$ for all $\dim(\b\Lambda) = j_0,\ldots,K-1$ and $m=1,\ldots, M_{\b\Lambda}$, it  needs at most
$
		2C\sum_{j=1}^{K-1} \left({ h^{K-j-1} \cdot h \cdot h^{j-1}} + 2h^{K-j-1} \cdot \f{h(h-1)}{2} \cdot h^{j-1} \right)
		=  2C(K-1)h^K = O(nh\log_h n)
$
evaluations of multiplications and additions.
	\end{proof}

Before showing the proof of \cref{thm4}, we want to give some comments on permutations. Notice that the construction of $\b p_{\b\Lambda}$ and $\b B_{\b\Lambda}$ only depends on $L_{\b\Lambda}$. Hence, under node permutation ($\pi$ or $\b P_{\pi}$), it means that we have the following relationship between original $\b\phi_{\b\Lambda}$ and $\b\phi^*_{\b\Lambda}$, $\b\psi_{\b\Lambda}$ and $\b\psi^*_{\b\Lambda}$,
\begin{align}\label{alg2}
	(\b\phi^*_{  {\b\Lambda}})^\top
	:=&
	\b p_{\b \Lambda}^\top
	\begin{bmatrix}
		\b\phi_{(\b\Lambda, 1)}^\top \\ \vdots \\ \b\phi_{{(\b\Lambda, L_{\b\Lambda})}}^\top
	\end{bmatrix}
	\b P_{\pi},\\
	\begin{bmatrix}
		(\b\psi^*_{(\b\Lambda,1)})^\top \\ \vdots \\ (\b\psi^*_{(\b\Lambda,M_{\b\Lambda})})^\top
	\end{bmatrix}
	:=&
	\b B_{\b\Lambda}
	\begin{bmatrix}
		\b\phi_{(\b\Lambda, 1)}^\top \\ \vdots \\ \b\phi_{{(\b\Lambda, L_{\b\Lambda})}}^\top
	\end{bmatrix}\b P_{\pi}.
\end{align}
Under partition permutation $\pi_p$, fixing a $\b\Lambda$ (there exists a permutation matrix $\b Q_{\b\Lambda}$ w.r.t. $\pi_p$ at $\b\Lambda$), we have the following relationship between original $\b\phi_{\b\Lambda}$ and $\b\phi^*_{\b\Lambda}$, $\b\psi_{\b\Lambda}$ and $\b\psi^*_{\b\Lambda}$,
\begin{align}\label{alg2}
	(\b\phi^*_{  {\b\Lambda}})^\top
	:=&
	\b p_{\b \Lambda}^\top \b Q_{\b\Lambda}
	\begin{bmatrix}
		\b\phi_{(\b\Lambda, 1)}^\top \\ \vdots \\ \b\phi_{{(\b\Lambda, L_{\b\Lambda})}}^\top
	\end{bmatrix}
	,\\
	\begin{bmatrix}
		(\b\psi^*_{(\b\Lambda,1)})^\top \\ \vdots \\ (\b\psi^*_{(\b\Lambda,M_{\b\Lambda})})^\top
	\end{bmatrix}
	:=&
	\b B_{\b\Lambda} \b Q_{\b\Lambda}
	\begin{bmatrix}
		\b\phi_{(\b\Lambda, 1)}^\top \\ \vdots \\ \b\phi_{{(\b\Lambda, L_{\b\Lambda})}}^\top
	\end{bmatrix}.
\end{align}

	\begin{proof}[Proof of \cref{thm4}]
		Let $\b\Phi_{ {{\b\Lambda}}}:=[\b\phi_{( {\b\Lambda},1)}, \dots, \b\phi_{( {\b\Lambda},L_{ {{\b\Lambda}}})}]^\top$ and $\b\Psi_{ {{\b\Lambda}}}:=[\b\psi_{( {\b\Lambda},1)}, \dots, \b\psi_{( {\b\Lambda},M_{ {{\b\Lambda}}})}]^\top$.
		Since the scaling vectors $\b\phi_{ {{\b\Lambda}}}^\top = \b p^\top_{ {{\b\Lambda}}} \b\Phi_{ {{\b\Lambda}}}$ are defined iteratively for $\dim(\b\Lambda)$  decreasing from $K$ to $1$ through \cref{1} and the framelets $\b\psi_{(\b\Lambda,m)}$ are given by $\b\Psi_{ {{\b\Lambda}}} = \b B_{ {{\b\Lambda}}}\b\Phi_{ {{\b\Lambda}}}$, we only need to prove the permutation equivariance properties for each $ {{\b\Lambda}}$.
		
For Item (i),  note that by \cref{1} and \cref{cor}, $\b\phi_{(\b\Lambda,\ell)}: \v\rightarrow \mathbb{R}$ only depends on $\g$, $\p_K$, and $\b p_{\b\Lambda}=\f{1}{\sqrt{L_{\b\Lambda}}}\b 1$.  For any node permutation $\pi$,  the $\p_K$ is determined by the index vectors $\b\Lambda$ according to a tree structure and is independent of the node permutation $\pi$.  Moreover, the vectors $\b p_{\b\Lambda}$ are fixed constants. Hence, iteratively, after node permutation $\per$ acting on graph $\g$,  the  new  scaling vector  $\b\phi_{(\b\Lambda,\ell)}^\pi:\pi(\v)\rightarrow \mathbb{R}$ is given by  $\b\phi_{(\b\Lambda,\ell)}^\pi=\b P_\pi\b\phi_{(\b\Lambda,\ell)}$, where $\b P_\pi$ is the permutation matrix with respect to $\pi$. Consequently,  the new $\b\Phi_{ {{\b\Lambda}}}^\pi$ and $\b\Psi^\pi_{ {{\b\Lambda}}}$ on the permuted graph $\pi(\g)$ are given by $\b\Phi^\pi_{ {{\b\Lambda}}} = \b\Phi_{ {{\b\Lambda}}} \b P_\pi$ and  $\b\Psi^\pi_{ {{\b\Lambda}}} = \b B_{ {{\b\Lambda}}} \b\Phi^\pi_{ {{\b\Lambda}}} = \b B_{ {{\b\Lambda}}} \b\Phi_{ {{\b\Lambda}}} \b P_\pi = \b\Psi_{ {{\b\Lambda}}} \b P_\pi$. This implies the conclusion in Item (i).

	For Item (ii), given a partition permutation  $\per_p$ acting on $\p_K$,  We denote $\per_p(\p_K)$ the hierarchical clustering w.r.t. such a $\per_p$.  Let  $\tilde{\b\Lambda}:=\pi_p(\b\Lambda)$ be the permuted index vector $\per_p(\p_K)$ from  the index vector  $\b\Lambda$ in $\p_K$. Since the partition permutation acts on the children of each $\b\Lambda$ only,  we have $\pi_p(\b\Lambda,\ell) = (\tilde{\b\Lambda},\pi_{\b\Lambda}(\ell))$ for some permutation $\pi_{\b\Lambda}$ on $[L_{\b\Lambda}]$. Then,		
 the matrix $\b\Phi_{ \tilde{{\b\Lambda}}}$ is 
\[
{\b\Phi}_{\tilde {{\b\Lambda}}}:=[\b\phi_{( \tilde{\b\Lambda},\pi_{\b\Lambda}(1))}, \dots, \b\phi_{( \tilde{\b\Lambda},\pi_{\b\Lambda}(L_{ {{\b\Lambda}}}))}]^\top=\b P_{\pi_{\b\Lambda}}\b\Phi_{ {{\b\Lambda}}}
\]  
with $\b P_{\pi_{\b\Lambda}}$ being the permutation matrix with respect to $\pi_{\b\Lambda}$.	Then, in view of  $\b p^\top_{ {{\b\Lambda}}} \b P_{\pi_{\b\Lambda}}=\b p^\top_{ {{\b\Lambda}}}$, the permuted scaling vector $\b\phi_{\tilde{\b\Lambda}}$  is given by
\[
\begin{aligned}
(\b\phi_{\tilde{\b\Lambda}})^\top  &= (\b\phi_{\pi_p({\b\Lambda})})^\top=\b p^\top_{ {{\b\Lambda}}} (\b P_{\pi_{\b\Lambda}}\b\Phi_{\b\Lambda})\\&= (\b p^\top_{ {{\b\Lambda}}} \b P_{\pi_{\b\Lambda}})\b\Phi_{\b\Lambda}= \b p^\top_{ {{\b\Lambda}}} \b\Phi_{\b\Lambda}\b = \b\phi_{\b\Lambda}^\top,
\end{aligned}
\]
That is, the new scaling vectors in $\{\b\phi_{\tilde{\b\Lambda}}:\dim({\tilde{\b\Lambda}})=j\}$ are simply the recording of $\{\b\phi_{{\b\Lambda}}:\dim({{\b\Lambda}})=j\}$ under $\pi_p$ for $j=0,\ldots,K$.  Thus, all scaling vectors are invariant (up to index permutation) under the partition permutation $\pi_p$.  Now for the framelet vectors $\b\psi_{(\b\Lambda,m)}$, by \cref{2}, we have 
\[
\b\Psi_{ {\tilde{\b\Lambda}}} =  \b B_{\b\Lambda}  \b\Phi_{ {\tilde{\b\Lambda}}} =  \b B_{ {{\b\Lambda}}}  \b P_{\pi_ {{\b\Lambda}}}  \b\Phi_{ {{\b\Lambda}}}.
\] 
We claim that  there exist $M_{\b\Lambda}\times M_{\b\Lambda}$ permutation matrix $\b R_{ {{\b\Lambda}}}$ and sign matrix $\b S_{ {{\b\Lambda}}}=\mathrm{diag}(c_1,\ldots,c_{M_{\b\Lambda}})$ with all $c_i\in\{-1,+1\}$  such that $\b B_{ {{\b\Lambda}}}\b P_{ \pi_{{\b\Lambda}}} = \b S_{ {{\b\Lambda}}}\b R_{ {{\b\Lambda}}}\b B_{ {{\b\Lambda}}}$. Then, we have 
\[ 
\b\Psi_{ {\tilde{\b\Lambda}}} =  \b B_{ {{\b\Lambda}}}  \b P_{\pi_ {{\b\Lambda}}}  \b\Phi_{ {{\b\Lambda}}} =  \b S_{ {{\b\Lambda}}}  \b R_{ {{\b\Lambda}}}  \b B_{ {{\b\Lambda}}}   \b\Phi_{ {{\b\Lambda}}} =  \b S_{ {{\b\Lambda}}}  \b R_{ {{\b\Lambda}}} \b\Psi_{ {{\b\Lambda}}},
\] 
which then concludes Item (ii). 
Noting that $\b B_{ {{\b\Lambda}}}  \b P_{ \pi_{{\b\Lambda}}}$ is to reorder the columns of $\b B_{ {{\b\Lambda}}}  $ and regardless the sign, all elements appear in each column with the same times and $1$ (or $-1$) appears in rows of $\b B_{\b\Lambda}$ once. In other words, $(\b B_{\b\Lambda} \b P_{\pi_{\b\Lambda}} )_{r:}= \b w^\top \b P_r^\top \b P_{\pi_{\b\Lambda}}$, which is to permute $\b w=[1,-1,0,\dots,0]^\top$ (up to a constant) with respect to  $\b P_r^\top \b P_{\pi_{\b\Lambda}}$. Since $\rank(\b B_{\b\Lambda})=L-1$ and $\b B_{\b\Lambda} \b 1 =\b 0$, we have $\b P \b w \in \Span \{\b P_m \b w \}_{m=1}^{M_{\b\Lambda}} $ for any permutation matrix $\b P$. Thus for any $r$, there exists exactly one $j\in[M_{\b\Lambda}]$ such that $(\b B_{\b\Lambda} \b P_{\pi_{\b\Lambda}} )_{r:}= \b w^\top \b P_r^\top\b P_{\pi_{\b\Lambda}}  = c \b w^\top \b P_j^\top$ where $c$ is either $1$ or $-1$. Hence the claim holds. This completes the proof of Item (ii).

The proof of Item (iii) is a direct consequence of Items (i) and (ii). 
	\end{proof}

\begin{proof}[Proof of \propref{prop:model}]
From Item (i) of \cref{thm4} , we see that the corresponding permuted versions of $\b{\Phi}_1$ and $\b{\Psi}_j$ are $\b{\Phi}_1\b{P}$ and $\b{\Psi}_j\b{P}$. Thus $\b{F}_j(\b{P}\b{X})\ab = \ab\b{P}\b{F}_j(\b{X})$, $\ab \b{F}_j(\b{P}\b{A}\b{P}^\top\b{P}\b{X})\ab = \ab \b{P}\b{F}_j(\b{A}\b{X})$, \ab $\b{F}_j(\b{P}\tilde{\b{A}}\b{P}^\top\b{P}\b{X})\ab =\ab \b{P}\b{F}_j(\tilde{\b{A}}\b{X})$ for $\ab j=0,\ab\dots,\ab K-1$. It is obvious that the remaining channels also differ by a permutation matrix $\b{P}$. Since the row normalization and the softmax function are applied row-wise and the activation function is applied element-wise, it is straightforward to see that $\hat{\b{Y}}_{\b{P}} = \b{P}\hat{\b{Y}}$.
\end{proof}	

%
%\begin{proof}[Proof of \cref{prop:model}]
%From \cref{thm4} (i), we know that the corresponding permuted versions of $\b{\Phi}_1$ and $\b{\Psi}_j$ are $\b{\Phi}_1\b{P}$ and $\b{\Psi}_j\b{P}$. Thus $\b{F}_j(\b{P}\b{X})\ab = \ab\b{P}\b{F}_j(\b{X})$, $\ab \b{F}_j(\b{P}\b{A}\b{P}^\top\b{P}\b{X})\ab = \ab \b{P}\b{F}_j(\b{A}\b{X})$, \ab $\b{F}_j(\b{P}\tilde{\b{A}}\b{P}^\top\b{P}\b{X})\ab =\ab \b{P}\b{F}_j(\tilde{\b{A}}\b{X})$ for $\ab j=0,\ab\dots,\ab K-1$. It is obvious that the remaining channels also differ by a permutation matrix $\b{P}$. Since the row normalization and the softmax function are applied row-wise and the activation function is applied element-wise, it is straightforward to see that $\hat{\b{Y}}_{\b{P}} = \b{P}\hat{\b{Y}}$.
%\end{proof}
		
%\section{Fast Decomposition and Reconstruction Algorithms}
	\subsection{Fast Decomposition and Reconstruction Algorithms}
	\label{app:alg}
	
	Given a $K$-hierarchical clustering $\p_K$, we consider graph Haar framelet transform between $V_{j+1}$ and $V_{j}\oplus W_{j}$. Define $x_{(\b\Lambda,\ell)}:= \langle \b f, \b\phi_{{(\b\Lambda, \ell)}} \rangle$ and $y_{(\b\Lambda,m)}:= \langle \b f , \b\psi_{  {(\b\Lambda,m)}} \rangle$ for a given graph signal $\b f$. The transform algorithm is to evaluate $x_{(\b\Lambda,\ell)}$ and $y_{(\b\Lambda,m)}$ effectively. Let $\b C_{ {{\b\Lambda}}} \in \RR^{L_{ {{\b\Lambda}}}\times (1+M_{ {{\b\Lambda}}})}$ be a matrix satisfying $\b C_{\b\Lambda} \binom{\b p_{ {{\b\Lambda}}}^\top}{\b B_{ {{\b\Lambda}}}}=\b I \in \RR^{L_{ {{\b\Lambda}}}\times L_{ {{\b\Lambda}}}}$. Then Lemma 1 in \cite{li2022convolutional} and \cref{1,2} imply that
	\begin{align}\label{ccc}
		\begin{bmatrix}
			\b \phi_{ {{\b\Lambda}}}^\top \\ \b\psi_{(\b\Lambda,1)}^\top \\ \vdots \\ \b\psi_{(\b\Lambda,M_{\b\Lambda})}^\top
		\end{bmatrix}
		:=&
		\binom{\b p_{ {{\b\Lambda}}}^\top}{\b B_{ {{\b\Lambda}}}}
		\begin{bmatrix}
			\b\phi_{(\b\Lambda, 1)}^\top \\ \vdots \\ \b\phi_{{(\b\Lambda, L_{\b\Lambda})}}^\top
		\end{bmatrix}, \\
		\begin{bmatrix}
			\b\phi_{(\b\Lambda, 1)}^\top \\ \vdots \\ \b\phi_{{(\b\Lambda, L_{\b\Lambda})}}^\top
		\end{bmatrix}
		:=& \b C_{\b \Lambda}
		\begin{bmatrix}
			\b \phi_{ {{\b\Lambda}}}^\top \\ \b\psi_{(\b\Lambda,1)}^\top \\ \vdots \\ \b\psi_{(\b\Lambda,M_{\b\Lambda})}^\top
		\end{bmatrix}
		.
	\end{align}
	For the decomposition algorithm, we are given a signal $\b f \in V_{j+1}$, which means that 
	\[\b f
	= \sum_{\dim( {{\b\Lambda}})=j}\sum_{\ell\in [L_{ {{\b\Lambda}}}]} x_{( {{\b\Lambda}},\ell)} \b\phi_{( {{\b\Lambda}},\ell)}.
	\] 
	By \cref{ccc}, we have
	\begin{align}\label{dec}
	\begin{tiny}
		\begin{aligned}
			\b f
			&= \sum_{\dim( {{\b\Lambda}})=j}\sum_{\ell\in [L_{ {{\b\Lambda}}}]} x_{( {{\b\Lambda}},\ell)} \b\phi_{( {{\b\Lambda}},\ell)} \\
			&= \sum_{\dim( {{\b\Lambda}})=j} \sum_{\ell\in [L_{ {{\b\Lambda}}}]} x_{( {{\b\Lambda}},\ell)}\left( \b (C_{\b\Lambda})_{\ell,1}\b\phi_{ {\b\Lambda}} + \sum_{m\in [M_{ {{\b\Lambda}}}]} (C_{\b\Lambda})_{\ell,m+1}\b\psi_{(  {\b\Lambda},m)} \right) \\
			&= \sum_{\dim( {{\b\Lambda}})=j}\b\phi_{ {\b\Lambda}} \sum_{\ell\in [L_{ {{\b\Lambda}}}]} x_{( {{\b\Lambda}},\ell)} (C_{\b\Lambda})_{\ell,1} 
			\\&\qquad+ \sum_{\dim( {{\b\Lambda}})=j} \sum_{m\in [M_{ {{\b\Lambda}}}]}\b\psi_{(  {\b\Lambda},m)}\sum_{\ell\in [L_{ {{\b\Lambda}}}]}  x_{( {{\b\Lambda}},\ell)} (C_{\b\Lambda})_{\ell,m+1} \\
			&= \sum_{\dim( {{\b\Lambda}})=j} x_{\b\Lambda} \b\phi_{ {\b\Lambda}} 
			+ \sum_{\dim( {{\b\Lambda}})=j} \sum_{m\in [M_{ {{\b\Lambda}}}]} y_{(\b\Lambda,m)} \b\psi_{(  {\b\Lambda},m)},
		\end{aligned}
		\end{tiny}
	\end{align}
	where we can represent decomposition of $\b f$ with respect to each $\b \Lambda$ as  
	\begin{align}\label{dec11}
		\begin{aligned}
			&\begin{bmatrix}
				x_{ {{\b\Lambda}}} ,& y_{(\b\Lambda,1)}, & \cdots ,& y_{(\b\Lambda,M_{\b\Lambda})}
			\end{bmatrix}\\
			= &
			\begin{bmatrix}
				x_{(\b\Lambda,1)}, & x_{(\b\Lambda,2)}, & \cdots, & x_{(\b\Lambda,L_{\b\Lambda})}
			\end{bmatrix}
			\b C_{\b\Lambda}.
		\end{aligned}
	\end{align}
	
	Conversely, if we have $\b f \in V_j \oplus W_j$, which means that 
	\[
	\b f = \sum_{\dim( {{\b\Lambda}})=j}x_{ {{\b\Lambda}}}
	\b\phi_{ {\b\Lambda}} + \sum_{\dim( {{\b\Lambda}})=j} \sum_{m\in [M_{ {{\b\Lambda}}}]} y_{( {{\b\Lambda}},m)} \b\psi_{(  {\b\Lambda},m)},
	\] 
	then,	 by \cref{ccc}, for the reconstruction from $V_{j}\oplus W_{j}$ to $V_{j+1}$, we have
	\begin{align}\label{rec}
		\begin{aligned}
			\b f &= \sum_{\dim( {{\b\Lambda}})=j}x_{ {{\b\Lambda}}}
			\b\phi_{ {\b\Lambda}} + \sum_{\dim( {{\b\Lambda}})=j} \sum_{m\in [M_{ {{\b\Lambda}}}]} y_{( {{\b\Lambda}},m)} \b\psi_{(  {\b\Lambda},m)}\\
			&=\sum_{\dim( {{\b\Lambda}})=j}  \sum_{\ell\in [L_{ {{\b\Lambda}}}]} \left(\b r_{ {{\b\Lambda}}}\right)_\ell \b\phi_{(  {\b\Lambda}, \ell)}\\
			&=\sum_{\dim( {{\b\Lambda}})=j}  \sum_{\ell\in [L_{ {{\b\Lambda}}}]} x_{(\b\Lambda,\ell)} \b\phi_{(  {\b\Lambda}, \ell)}		,
		\end{aligned}
	\end{align}
	where
	\begin{align}\label{rec22}
		\b r_{ {{\b\Lambda}}} = x_{ {{\b\Lambda}}}
		\b p_{ {{\b\Lambda}}}^\top +   \b y_{ {{\b\Lambda}}}^\top \b B_{ {{\b\Lambda}}}, \mbox{ with }
		\b y_{ {{\b\Lambda}}}  = [y_{( {{\b\Lambda}},1)},\dots, y_{( {{\b\Lambda}},M_{ {{\b\Lambda}}})}]^\top,
	\end{align} 
	and	 $x_{(\b\Lambda,\ell)} := \b (r_{\b\Lambda})_\ell$.
	
	\begin{algorithm}[htpb!]
		\caption{Fast framelet decomposition}
		\label{alg:dec}
		\begin{algorithmic}
			\State {\bfseries Input:} $\p_K$, $\{ x_{\b\Lambda}:   \dim(\b\Lambda) = j_0 \}$, $\{\b C_{\b\Lambda}\}$, $j_1$
			\State initialize $\hat{\b f}=\emptyset$.
			\For{$j=j_0-1$ {\bfseries to} $j_1$}
			\For{$\b\Lambda \in  \{\b\Lambda:\dim(\b\Lambda)=j\}$}
			\State $[x_{ {{\b\Lambda}}} , y_{(\b\Lambda,1)},  \cdots , y_{(\b\Lambda,M_{\b\Lambda})}]
			\leftarrow 
			[x_{(\b\Lambda,1)},  x_{(\b\Lambda,2)},  \cdots,  x_{(\b\Lambda,L_{\b\Lambda})}]
			\b C_{\b\Lambda}$
			\State update $\hat{\b f} \leftarrow \hat{\b f} \cup \{x_{  {\b\Lambda}},y_{(  {\b\Lambda}, m)},m=1,\dots,M_{\b\Lambda} \}$
			\EndFor
			\EndFor
			\State {\bfseries Output:} $\mathcal{F}_{j_0}(\p_K)$
		\end{algorithmic}
	\end{algorithm}
	
	\begin{algorithm}[htpb!]
		\caption{Fast framelet reconstruction}
		\label{alg:rec}
		\begin{algorithmic}
			\State {\bfseries Input:} $\p_K$, $\{ x_{\b\Lambda}:   \dim(\b\Lambda) = j_0 \} \cup \{ y_{(\b\Lambda,m)}:\dim(\b\Lambda) = j_0, m\in [M_{\b\Lambda}]  \}$, $\{\b p_{\b\Lambda}, \b C_{\b\Lambda}\}$, $j_1$
			\State initialize $\b f=\emptyset$.
			\For{$j=j_0+1$ {\bfseries to} $j_1$}
			\For{$\b\Lambda \in  \{\b\Lambda:\dim(\b\Lambda)=j\}$}
			\State $\b r_{ {{\b\Lambda}}} = x_{ {{\b\Lambda}}}
			\b p_{ {{\b\Lambda}}}^\top +   \b y_{ {{\b\Lambda}}}^\top \b B_{ {{\b\Lambda}}}$
			\For{$\ell=1$ {\bfseries to} $L_{\b\Lambda}$}
			\State $x_{(\b\Lambda,\ell)} \leftarrow (\b r_{\b\Lambda})_\ell$
			\EndFor
			\State update $\b f \leftarrow \b f \cup \{x_{  (\b\Lambda,\ell) } \}_{\ell \in [L_{\b\Lambda}]}$
			\EndFor
			\EndFor
			\State {\bfseries Output:} $\b f$
		\end{algorithmic}
	\end{algorithm}
	
	Hence, by using \cref{dec11} iteratively from $V_K$, given a framelet system $\mathcal{F}_{j_0}(\p_K)$ and a graph signal $\b f$,  we get the coefficient vector $\hat {\b f}$ consisting of coefficients from 
	\begin{equation}
		\label{f-to-hat-f}
		\b f\mapsto \{x_{\b\Lambda}: \dim(\b\Lambda)= j_0\}\cup \{y_{\b\Lambda} :\dim(\b\Lambda) = j\}_{j=j_0+1}^{K-1}
	\end{equation}
	with respect to $V_{j_0} \oplus W_{j_0} \oplus \cdots \oplus W_{K-1}$. In the reconstruction process, we iteratively obtain the representation of $\b f$ in $V_K$ from coefficient vector $\hat {\b f}$:
	\begin{equation}
		\label{f-from-hat-f}
		\{x_{\b\Lambda}: \dim(\b\Lambda)= j_0\}\cup \{y_{\b\Lambda} :\dim(\b\Lambda) = j\}_{j=j_0+1}^{K-1}\mapsto {\b f}
	\end{equation}
	with respect to  $V_{j_0} \oplus W_{j_0} \oplus \cdots \oplus W_{K-1}$.
	
	From the above discussion, we observe that decomposition and reconstruction algorithms do not need to form the full framelet system explicitly, but $\b p_{ {{\b\Lambda}}}$, $\b B_{ {{\b\Lambda}}}$ and $\b C_{ {{\b\Lambda}}}$, which implies efficiency in general applications that apply our framelet system.

	\begin{theorem}\label{fast_alg}
		Under the same assumption as in \cref{cor}.  The decomposition algorithm to obtain the framelet coefficient vector $\hat{\b f}$ from $\b f$  and the reconstruction algorithm to obtain the graph signal $\b f$ from $\hat{\b f}$, as described in \cref{f-to-hat-f,f-from-hat-f}, are both with a computational complexity of order $O(nh)$.
	\end{theorem}
	\begin{proof}[Proof of \cref{fast_alg}]
		A fast decomposition algorithm is given by \cref{dec}, which computes $\hat{\b f}$ iteratively. In order to get $x_{\b\Lambda}$ and $y_{(\b\Lambda,m)}$ for $\hat {\b f}$, we need to compute $\sum_{\ell\in [L_{ {{\b\Lambda}}}]} x_{( {{\b\Lambda}},\ell)} (\b C_{\b\Lambda})_{\ell, t} = (\b q_{\b\Lambda}^\top\b C_{\b\Lambda})_t$, where $t = 1, \dots, M_{\b\Lambda}+1, \ell \in [L_{\b\Lambda}]$ and  $\b q_{\b\Lambda}:= [x_{(\b\Lambda,1)},\dots, x_{(\b\Lambda,L_{\b\Lambda})}]^\top$ (see \cref{dec11}). Note that for our binary graph Haar framelet system $\mathcal{F}_{j_0}(\p_K)$, the matrix $\b  C_{\b\Lambda}$ in \cref{ccc} is given by $\b  C_{\b\Lambda} = [\b p_{\b\Lambda}, \b B_{\b\Lambda}^\top]$ and each row of $\b B_{\b\Lambda}$ has only two nonzero elements. Hence, for a given $\b\Lambda$ with $\dim(\b\Lambda)=j$, since $L_{\b\Lambda}\leq h$ and $M_{\b\Lambda}\leq \f{h(h-1)}{2}$, the number of nonzero elements in $\b C_{\b\Lambda}$ is no more than $ h + 2\cdot \f{h(h-1)}{2}$. Therefore, the computational complexity for obtaining $\b q _{\b\Lambda}^\top \b C$ is of the same order as $h + 2\cdot \f{h(h-1)}{2}$. In total, observing that $\#\{\b\Lambda:\dim(\b\Lambda)=j \} \leq h^{j-1}$, to get the full $\hat{\b f}$, the computational complexity is of order the same as $\sum_{j=1}^{K-1} (h + 2 \cdot \f{h(h-1)}{2}) h^{j-1}  = O(nh)$.
		
		Fast reconstruction algorithm (\cref{rec}) which computes $\b f$ from $\hat{\b f}$ only need to compute $\b r_{\b\Lambda}$ iteratively. Let $\b Y_{\b\Lambda}:=[x_{\b\Lambda}, \b y_{\b\Lambda}^\top]^\top \in \RR^{M_{\b\Lambda}+1}$ and $\b P_{\b\Lambda}:= \binom{\b p_{ {{\b\Lambda}}}^\top}{\b B_{ {{\b\Lambda}}}}$. Then $\b P_{\b\Lambda} = \b C_{\b\Lambda}^\top$ and $\b r_{\b\Lambda} = \b Y_{\b\Lambda}^\top \b P_{\b\Lambda} $. Following a similar calculation as for the fast decomposition algorithm, it is not hard to see the computation complexity is of $\sum_{j=1}^{K-1} h^{j-1} (h + 2\cdot \f{h(h-1)}{2} ) = O(nh)$.	
	\end{proof}

	\subsection{Experiment Details on the Synthetic Dataset}
	\label{appen_d}
	%	\begin{figure*}[t]
	%		\centering
	%		%\fbox{\rule{0pt}{2in} \rule{0.9\linewidth}{0pt}}
	%		\includegraphics[width=0.8\linewidth]{figures/line.pdf}
	%		
	%		\caption{Test accuracy.}
	%		\label{fig:nn}
	%	\end{figure*}
	
	\begin{figure*}[ht]\label{CCNS}
		\centering
		\subfloat[$\gamma = 0$]{\includegraphics[width=0.3\linewidth]{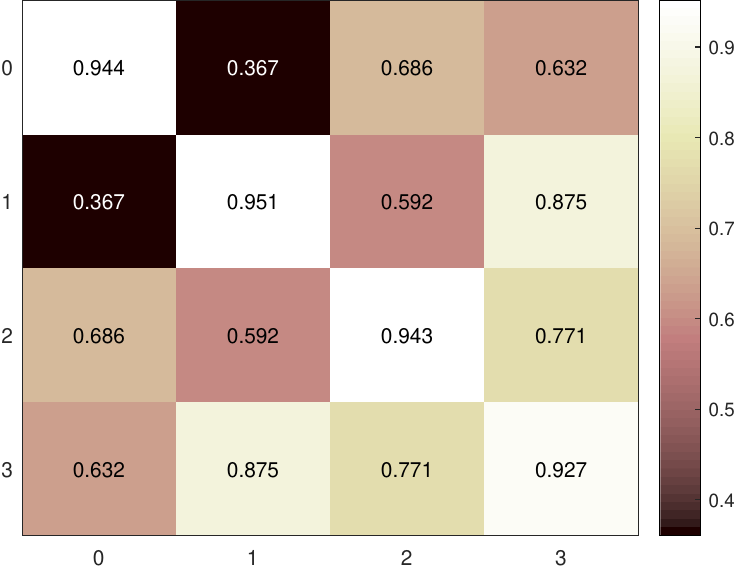}}
		\subfloat[$\gamma = 0.2$]{\includegraphics[width=0.3\linewidth]{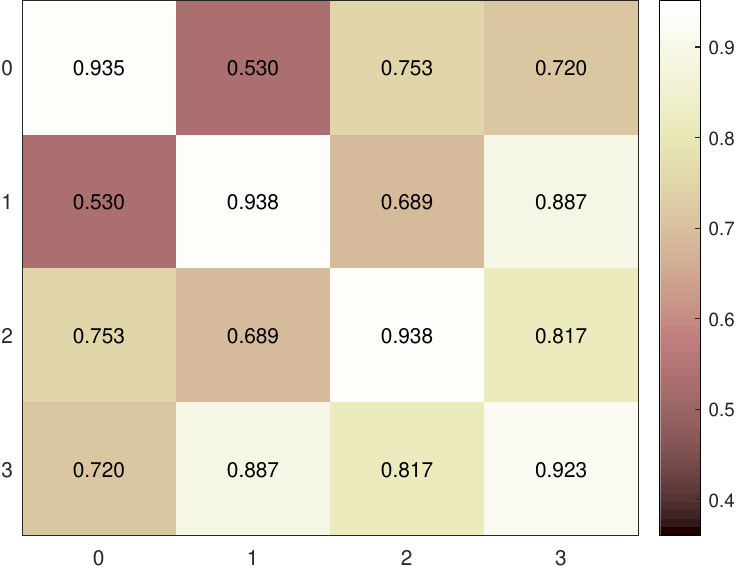}}
		\subfloat[$\gamma = 0.4$]{\includegraphics[width=0.3\linewidth]{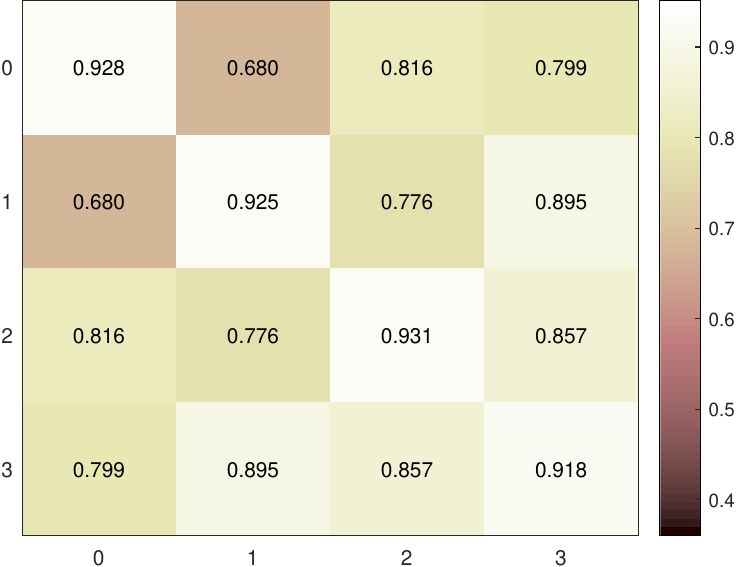}}
		\\
		\subfloat[$\gamma = 0.6$]{\includegraphics[width=0.3\linewidth]{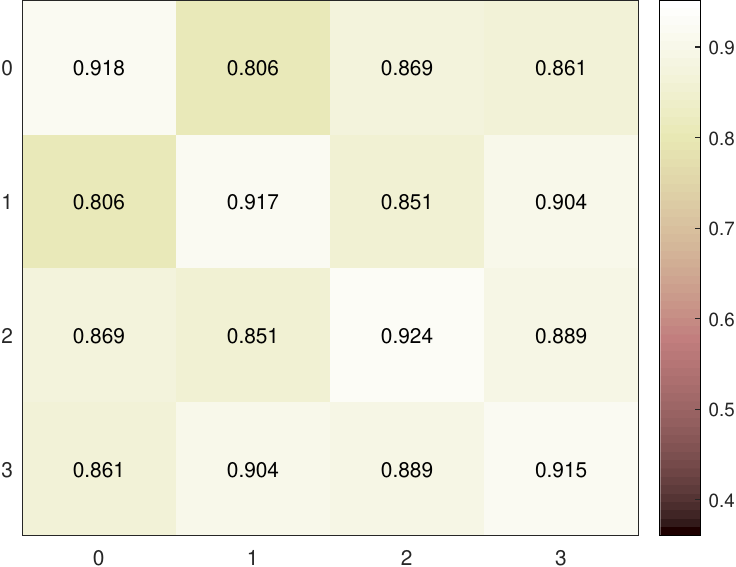}}
		\subfloat[$\gamma = 0.8$]{\includegraphics[width=0.3\linewidth]{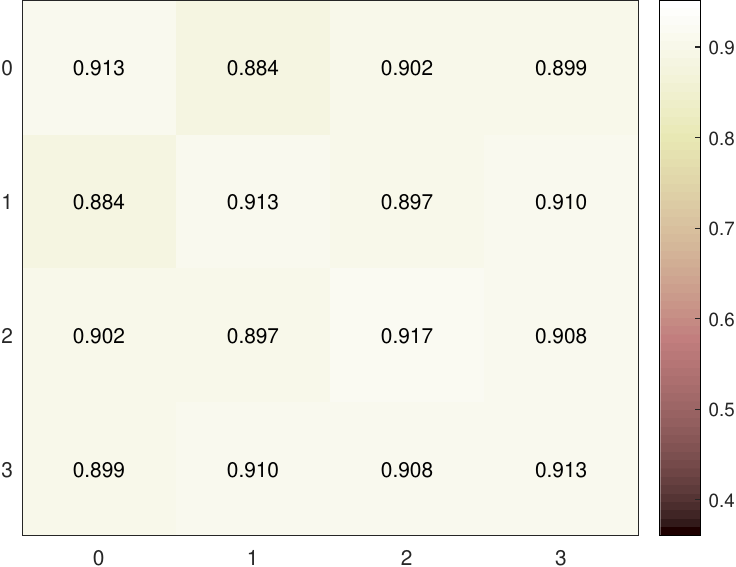}}
		\subfloat[$\gamma = 1$]{\includegraphics[width=0.3\linewidth]{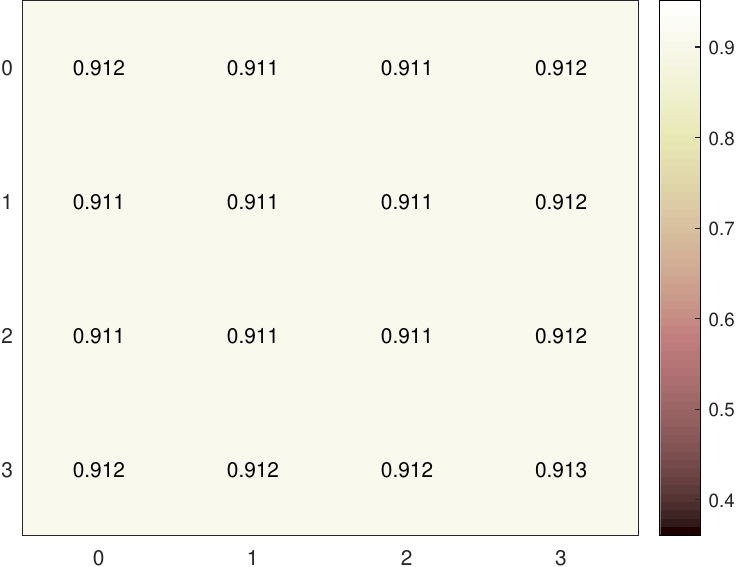}}
		
		\caption{CCNS on the synthetic graphs with different hyperparameters $\gamma$.}
		\label{heatmap}
	\end{figure*}
	
	\cite{ISGCN} gives a theoretical characterization of graphs on which GCN fails to produce acceptable performance. We follow and modify Algorithm 2 in  \cite{ISGCN} to generate synthetic data.
	
	The key idea of the algorithm is to generate edges of nodes in a graph such that the intra-class and inter-class similarities are properly controlled. The intra-class and inter-class similarity are defined by \cref{CCNS}. Cross-class neighborhood similarity (CCNS) measures how close the patterns of connections of nodes between two classes are. In our experiment, we generate graphs with $3,000$ nodes for which assign labels from $\c = \{0,1,2,3\} $ randomly. It means that we have $n_c=4$ classes. We generate edges according to \cref{alg:example}. The distributions that control CCNS are designed based on the uniform distribution and a prescribed distribution $\{\d_{c}: c \in \c \}$. The distributions $\{\d_{c}: c \in \c \}$ can be found in \cref{DC}.  Integer $N$ is set to be $45,000$.
	The hyperparameter $\gamma \in \{0,0.2,0.4,0.6,0.8,1\}$ indicates the probability of sampling neighbors from uniform distributions other than the predefined distributions that are much more distinguishable for different classes. When $\gamma$ is small,  vertices of the neighborhood are more likely to be sampled according to $\{\d_{c}: c \in \c \}$ and when $\gamma$ is large, it is more likely to sample from the indistinguishable uniform distribution. As a result, when $\gamma$ becomes larger, the uniform distribution has more impact on CCNS and thus the metric becomes more similar between classes.
%	from uniform distribution, when $\gamma$ is small, distributions of labels in the neighborhood of $i\in \v$ are close to $\{\d_{c}: c \in \c \}$ and when $\gamma$ is large, those distributions are close to uniform distribution.
%	The number $\gamma \in \{0,0.2,0.4,0.6,0.8,1\}$ is to control the probability from uniform distribution when $\gamma$ is small, distributions of labels in the neighborhood of $i\in \v$ are close to $\{\d_{c}: c \in \c \}$ and when $\gamma$ is large, those distributions are close to uniform distribution.
	We evaluate CCNSs on the generated graphs and show heatmaps in \cref{heatmap}. It is clear that, when $\gamma=0$, CCNS is dominated by $\{\d_{c}: c \in \c \}$, and since $\d_{0}$ is more similar to $\d_{2}$ than $\d_{1}$, we get $s(0,1)=0.367 \leq 0.686 = s(0,2)$.  And finally when $\gamma = 1$, all CCNS are almost $0.91$. Notice that in \cref{alg:example}, we slightly modify the algorithm in \cite{ISGCN}. Since we generate graphs with only nodes initialized, when $r\leq \gamma$, we sample label $c$ from all labels $\c$, instead of $\c - \{y_i \}$ used in \cite{ISGCN}. \cref{HP} shows the hyperparameter search range for the experiments of FSGNN and PEGFAN on the synthetic data, where $\{WD_{sca},LR_{sca},WD_{fc1},WD_{fc2},LR_{fc}\}$ are the weight decay coefficient of attention weights, learning rate of attention weights, weight decay coefficient of the first fully-connected layer, weight decay coefficient of the second fully-connected layer and learning rate of the fully-connected part, respectively.
	
	\begin{definition}[Cross-Class Neighborhood Similarity (CCNS) \cite{ISGCN}]\label{CCNS}
		Given graph $\g$ and node labels $y_i \in \{0,1,\dots, n_c-1\}$ for $i \in \v$. the metric between classes $c$ and $c'$ is $s(c,c') = \f{1}{|\v_c||\v_{c'}|}\sum_{i\in \v_c, j \in \v{c'}} \cos \langle d(i), d(j) \rangle $, where $\v_c := \{ i\in \v: y_i= c \}$ and $d(i) \in \RR^{n_c}$ is a vector with elements defined by $ \# \{ j: (i,j)\in \w, y_j= c \} $ for any $c\in \{0,1,\dots, n_c-1\}$.
	\end{definition}

	\begin{table}[h]
		\caption{Distribution $D_4$}
		\centering
		\begin{tabular}{c|l|l|l|l}
			\toprule
			$\{\d_{c}: c \in \c \}$ & class 0 & class 1 & class 2 & class 3  \\
			\midrule
			$\d_0$ & 0.1 & 0.4 & 0 & 0.5 \\
			$\d_1$ & 0.5 & 0 & 0.5 & 0 \\
			$\d_2$ & 0.2 & 0 & 0.5 & 0.3 \\
			$\d_3$ & 0.25 & 0.25 & 0.25 & 0.25 \\
			\bottomrule
		\end{tabular}
		
		\label{DC}
	\end{table}
\begin{table}[htpb!]
\caption{Hyperparameter search range}
\centering
\begin{tabular}{l l}
\toprule
Hyperparameters & Values\\ \midrule
$WD_{sca}$&{$0.0,0.001,0.1$}\\
$LR_{sca}$&{$0.01,0.04$}\\
$WD_{fc1}$&{$0.0,0.0001,0.001$}\\
$WD_{fc2}$&{$0.0,0.0001,0.001$}\\
$LR_{fc}$&{$0.005,0.01$}\\
\bottomrule
\end{tabular}
\label{HP}
\end{table}
	
	Features on graph nodes are from $\RR^{700}$, with each element randomly generated according to Gaussian distribution $ (- \f{9}{2} + \f{1}{2}c)+\xi$ (Table \ref{tab:normal1}) or $ (- \f{3}{4} + \f{1}{2}c)+\xi$ (Table \ref{tab:normal2}) independently, where $\xi \sim N(0,1)$ and $c$ is the label.

	%	\begin{algorithm}[tb]
	%		\caption{\cite{ISGCN}}
	%		\label{alg:example}
	%		\begin{algorithmic}
	%			\State {\bfseries Input:} Nodes $\v$, Integer $K$, Distribution matrix $D_{n_c}$, labels $\{c\}_{i=0}^{n_c-1}$
	%			\State initialize $\w = \{\}$ and $k = 1$;
	%			\While {}%{$k \leq K$}
	%			\State Initialize $noChange = true$.
	%			\For{$i=1$ {\bfseries to} $m-1$}
	%			\If{$x_i > x_{i+1}$}
	%			\State Swap $x_i$ and $x_{i+1}$
	%			\State $noChange = false$
	%			\EndIf
	%			\EndFor
	%			\UNTIL{$noChange$ is $true$}
	%		\end{algorithmic}
	%	\end{algorithm}
	
	\begin{algorithm}[htpb!]
		\caption{\cite{ISGCN}}
		\label{alg:example}
		\begin{algorithmic}
			\State {\bfseries Input:} Nodes $\v$, Integer $N$, Distribution matrix $\{\d_{c}: c \in \c \}$, labels $\c = \{c\}_{i=0}^{n_c-1}$, $\gamma$
			\State initialize $\w = \emptyset$ and $k = 0$;
			\While{$k \leq N$ }
			\State Sample $i \in \v$ and $r \in [0,1]$ uniformly
			\State Obtain the label $y_i \in \c$ of node $i$
			\If{$r \leq \gamma$}
			\State Sample a label $c$ from $\c$ uniformly
			\Else
			\State Sample a label $c$ from $\c$ according to distribution $\d_{y_i}$
			\EndIf
			\State Sample node $j$ from class $c$ uniformly
			\If {$(i,j)\notin \w$}
			\State update $\w \leftarrow \w \cup (i,j)$
			\State update $k \leftarrow k+1$
			\EndIf
			\EndWhile
			\State {\bfseries Output:} $\g = (\v,\w)$
		\end{algorithmic}
	\end{algorithm}

\end{appendix}

\bibliography{egbib1}
\bibliographystyle{IEEEtran}

\vfill

\end{document}